%% file: manuscript_arxiv.tex
\documentclass{article}
\usepackage{ICLR/iclr_preprint,times}

\input{ICLR/math_commands.tex}

\usepackage[utf8]{inputenc}
\usepackage{amsmath}
\usepackage{amssymb}
\usepackage{bm}
\usepackage[hidelinks]{hyperref}
\usepackage[capitalise]{cleveref}
\usepackage{xcolor}
\usepackage{graphicx}
\usepackage{subfigure}
\usepackage{enumitem}
\usepackage{chngcntr}
\usepackage{amsthm}
\newtheorem{theorem}{Theorem}
\newtheorem{corollary}{Corollary}
\newtheorem{lemma}{Lemma}
\newtheorem{proposition}{Proposition}

\DeclareMathOperator{\gl}{GL}

\newcommand{\cyc}[1]{Z_{#1}}

\author{Matthew Farrell\thanks{These authors contributed equally.}\text{  \,}$^{\dagger}$$^{\ddagger}$, \text{ } Blake Bordelon$^{*}$$^{\dagger}$$^{\ddagger}$, \text{ } Shubhendu Trivedi$^{\S}$, \& Cengiz Pehlevan$^{\dagger}$$^{\ddagger}$  \\
\\
$^{\dagger}$John A. Paulson School of Engineering and Applied Sciences\\
$^{\ddagger}$Center for Brain Science\\
\, Harvard University\\
\, Cambridge, MA 02138, USA \\
\,\texttt{\{msfarrell,blake\_bordelon,cpehlevan\}@seas.harvard.edu} \\
\\
$^{\S}$Massachusetts Institute of Technology \\
\, Cambridge, MA 02139, USA \\
\, \texttt{shubhendu@csail.mit.edu}
}

\title{Capacity of Group-invariant Linear Readouts from Equivariant Representations: How Many Objects can be Linearly Classified Under All Possible Views?}

\begin{document}
\maketitle
\begin{abstract}

Equivariance has emerged as a desirable property of representations of objects subject to identity-preserving transformations that constitute a group, such as translations and rotations. However, the expressivity of a representation constrained by group equivariance is still not fully understood. We address this gap by providing a generalization of Cover's Function Counting Theorem that quantifies the number of linearly separable and group-invariant binary dichotomies that can be assigned to equivariant representations of objects. We find that the fraction of separable dichotomies is determined by the dimension of the space that is fixed by the group action. We show how this relation extends to operations such as convolutions, element-wise nonlinearities, and global and local pooling. While other operations do not change the fraction of separable dichotomies, local pooling decreases the fraction, despite being a highly nonlinear operation. Finally, we test our theory on intermediate representations of randomly initialized and fully trained convolutional neural networks and find perfect agreement.

%Perceptron capacity, as quantified by Cover's Function Counting Theorem, plays a fundamental role \cep{a bit vague, maybe be more explicity e.g in quantifying expressivity?} in current understanding of successful machine learning with intermediate representations . However, work has only recently begun probing the perceptron capacity of structured representations. Here we consider representations with structure endowed by group actions: equivariant representations. Group actions naturally reflect many identity-preserving transformations of objects such as translations and rotations. We find simple relationships between the capacity of group-invariant perceptrons and the group structures themselves, and show how this extends to operations such as convolutions, element-wise nonlinearities, and local pooling. Precisely, we prove that the capacity of equivariant representations is determined by the dimension of the space that remains invariant to the group action \cep{This sentence could come before extension to pooling etc}. Finally, we test our theory by measuring the capacity of group-invariant linear readouts from intermediate representations of randomly initialized and fully trained convolutional neural networks.
    
\end{abstract}

\section{Introduction}
\vspace{-.1in}

The ability to robustly categorize objects under conditions and transformations that preserve the object categories
is essential to animal intelligence, and to pursuits of practical importance such as improving computer vision systems.
However, for general-purpose understanding and geometric reasoning, invariant representations of these objects in sensory processing circuits are not enough.
Perceptual representations must also accurately encode their transformation properties. 

One such property is that of exhibiting equivariance to transformations of the object.
When such transformations are restricted to be an algebraic %%%% after removal of rotations etc, which served as illustrative examples, it becomes important to specify the group types, since GCNNs can only work with locally compact groups.
group, 
%such as translations and rotations %%%% removed it as the next line repeats this.
the resulting equivariant representations have found significant success in machine learning starting with classical convolutional neural networks (CNNs)~\citep{NIPS1988_a97da629,LeCun1989BackpropagationAT} and recently being generalized by the influential work of~\citet{Cohen2016GroupEC}.
Such representations have elicited burgeoning interest as they capture many transformations of practical interest such as translations, permutations, rotations, and reflections. Furthermore, equivariance to these transformations can be easily ``hard-coded'' into neural networks. Indeed, a new breed of CNN architectures that explicitly account for such transformations are seeing diverse and rapidly growing applications ~\citep{Townshend2021GeometricDL,Baek2021AccuratePO, Satorras2021EnEN, andersonCormorantCovariantMolecular2019, BogatskiyLorentz2020, Klicpera2020DirectionalMP, Winkels2019PulmonaryND, Gordon2020PermutationEM, Sosnovik2021ScaleEI, Eismann2020HierarchicalRN}.
In addition, equivariant CNNs have been shown to capture response properties of neurons in the primary visual cortex beyond classical  G\'abor filter models \citep{ecker2018rotation}. 
%As objects transform in the world, the representations of these objects in sensory processing circuits transform with them, tracing out perceptual manifolds. Performing tasks requires accounting for this variation; for instance, accurate categorization of objects requires an understanding of which transformations leave object identity invariant. The ability to categorize objects in the presence of category-preserving transformations is essential to animal intelligence, and to pursuits of practical importance such as improving computer vision.

%\cep{This should be an equivariance paragraph, i.e. why equivariance is a good way to represent objects and their transformations.}

While it is clear that equivariance imposes a strong constraint on the geometry of representations and thus of perceptual manifolds~\citep{seung2000manifold,dicarlo2007untangling} that are carved out by these representations as the objects transform, the implications of such constraints on their expressivity are not well understood. In this work we take a step towards addressing this gap.
Our starting point is the classical notion of the \newterm{perceptron capacity} (sometimes also known as the fractional memory/storage capacity) -- a quantity fundamental to the task of object categorization and closely related to VC dimension~\citep{Vapnik1971ChervonenkisOT}.
Defined as the maximum number of points for which all (or 1-$\delta$ fraction of) possible binary label assignments (i.e. \newterm{dichotomies}) afford a hyperplane that separates points with one label from the points with the other,
% \footnote{The perceptron capacity is bounded above by the VC Dimension~\citep{Vershynin2020MemoryCO}.
%This is because the VC dimension is the maximal $K$ such that there exist $K$ points, so that for any binary labeling, there exists a function to realize them. The perceptron capacity requires any general set of $K$ points to be realizable as above.
% }
it can be seen to offer a quantification of the expressivity of a representation.
% To apply this theory to equivariant representations, we use a generalization of this notion that is applicable to manifolds
% (not just points) that have structure dictated by a group of transformations.

Classical work on perceptron capacity focused on points in general position  \citep{Wendel1962API, Cover1965GeometricalAS, Schlafli1950,Gardner1987MaximumSC, Gardner1988TheSO}. However, understanding the perceptron capacity when the inputs are not merely points, but are endowed with richer structure, has only recently started to attract attention.
For instance, work by ~\citet{chungClassificationGeometryGeneral2018, pastoreStatisticalLearningTheory2020, rotondoCountingLearnableFunctions2020, cohenSeparabilityGeometryObject2020} considered general perceptual manifolds and examined the role of their geometry to obtain extensions to the perceptron capacity results. However, such work crucially relied on the assumption that each manifold is oriented randomly, a condition which is strongly violated by equivariant representations.

With these motivations, our particular contributions in this paper are the following: %\cep{This list requires some more work}:
\begin{itemize}[leftmargin=*]
\item We extend Cover's function counting theorem and VC dimension to equivariant representations, finding
% to group-structured data, namely object manifolds that are equivariant to actions of arbitrary compact groups%We extend Cover's theorem to the setting where the  data is subject to a set of prescribed transformations, encapsulated by actions of arbitrary compact groups
% \footnote{We also deal with the $\mathbb{R}^d$ groups (i.e. translations) which are not compact, but they are special cases for which there is still a canonical measure, namely the Lebesgue
% measure. So for our purposes, we can consider them together with compact groups}.
% We find
% that the capacity of linearly separable dichotomies of manifolds generated by linear group actions is simply given in terms of
% the dimension of the space that is invariant to these actions.
that both scale with the dimension of the subspace fixed by the group action.
%We find that the number of “degrees of freedom” available for solving the classification task is the dimension of the space of points that are fixed by the group action
    
    %We extend Cover's theorem  to group-structured data, namely object manifolds that are equivariant to group actions. \cep{Discuss if this is the best way to state this}
    
    \item We demonstrate the applicability of our result to $G$-convolutional network layers, including pooling layers, through theory and verify through simulation.
    % equivariant representations.
    % This includes intermediate representations of traditional convolutional neural networks where input images are shown at every possible translation, as well as more general group convolutional neural networks (GCNNs) ~\citep{Cohen2016GroupEC, Kondor2018OnTG}.

    % \item Through simulations, we demonstrate a close match between our theory and representations obtained at each layer in equivariant convolutional networks i.e. output of group convolution followed by a point-wise non-linearity.
    %We demonstrate in simulations that representations formed by the output of equivariant convolutional layers with element-wise nonlinearities closely match our theory.
    % \item We show that the same counting relations hold up to a constant for local pooling.
    % \item We connect our theory to global average pooling as used in the \mc{finish}
    % \item We show an efficient method for learning an optimal linear classifier from knowledge of the group along with a single point from each equivariant object manifold.
    % We show that our theory can also be made to bear on local pooling operations, ubiquitous in deep convolutional neural network models. We report a good agreement of our theory with simulations.
\end{itemize}
%These findings extend Cover's theorem to equivariant representations in the situation where input data can take on a set of prescribed transformations given by group actions.

\subsection{Related works}
\vspace{-.1in}

Work most related to ours falls along two major axes. The first follows the classical perceptron capacity result on the linear separability of points ~\citep{Schlafli1950,Wendel1962API,Cover1965GeometricalAS, Gardner1987MaximumSC, Gardner1988TheSO}. This result initiated a long history of investigation in theoretical neuroscience, e.g. ~\citep{brunel2004optimal, Chapeton2012EfficientAM, rigotti2013importance, Brunel2016IsCC, Rubin2017BalancedEA, Pehlevan2017ResourceefficientPH}, where it is used to understand the memory capacity of neuronal architectures. Similarly, in machine learning, the perceptron capacity and its variants, including notions for multilayer perceptrons, have been fundamental to a fruitful line of study in the context of finite sample expressivity and generalization ~\citep{Baum1988OnTC,Kowalczyk1997EstimatesOS, Sontag1997ShatteringAS, Huang2003LearningCA, Yun2019SmallRN, Vershynin2020MemoryCO}. Work closest in spirit to ours comes from theoretical neuroscience and statistical physics ~\citep{chungClassificationGeometryGeneral2018, pastoreStatisticalLearningTheory2020, rotondoCountingLearnableFunctions2020, cohenSeparabilityGeometryObject2020}, which considered general perceptual manifolds, albeit oriented randomly, and examined the role of their geometry to obtain extensions to the perceptron capacity result. 

The second line of relevant literature is that on group-equivariant convolutional neural networks (GCNNs). The main inspiration for such networks comes from the spectacular success of classical CNNs ~\citep{LeCun1989BackpropagationAT} which directly built in translational symmetry into the network architecture. In particular, the internal representations of a CNN are approximately\footnote{Some operations such as max pooling and boundary effects of the convolutions technically break strict equivariance, as well as the final densely connected layers.} translation equivariant: if the input image is translated by an amount $t$, the feature map of each internal layer is translated by the same amount. Furthermore, an invariant read-out on top ensures that a CNN is translation invariant. ~\citet{Cohen2016GroupEC} observed that a viable approach to generalize CNNs to other data types could involve considering equivariance to more general transformation groups. This idea has been used to construct networks equivariant to a wide variety of transformations such as planar rotations~\citep{Worrall2017HarmonicND, Weiler2018LearningSF, Bekkers2018RotoTranslationCC, Veeling2018RotationEC, Smets2020PDEbasedGE}, 3D rotations~\citep{Cohen2018SphericalC,Esteves2018LearningSE, Worrall2018CubeNetET, Weiler20183DSC, Kondor2018ClebschGordanNA, Perraudin2019DeepSphereES}, permutations~\citep{Zaheer2017DeepS, Hartford2018DeepMO, Kondor2018CovariantCN, Maron2019InvariantAE, Maron2020OnLS}, general Euclidean isometries~\citep{Weiler20183DSC, Weiler2019GeneralES, Finzi2020GeneralizingCN}, scaling~\citep{Marcos2018ScaleEI, Worrall2019DeepSE, Sosnovik2020ScaleEquivariantSN} and more exotic symmetries~\citep{BogatskiyLorentz2020, Shutty2020LearningIR, Finzi2021APM} etc.

A quite general theory of equivariant/invariant networks has also emerged. ~\citet{Kondor2018OnTG} gave a complete description of GCNNs for scalar fields on homogeneous spaces of compact groups. This was generalized further to cover the steerable case in~\citep{cohenGeneralTheoryEquivariant2019} and to general gauge fields in~\citep{Cohen2019GaugeEC, Weiler2021CoordinateIC}. This theory also includes universal approximation results~\citep{Yarotsky2018UniversalAO, KerivenUniversal2019, Sannai2019UniversalAO, maron19a, Segol2020OnUE, ravanbakhsh20a}. Nevertheless, while benefits of equivariance/invariance in terms of improved sample complexity and ease of training are quoted frequently, a firm theoretical understanding is still largely missing.
% Benefits in terms of improved sample complexity and ease of training are quoted frequently in the literature, but are most often justified using intuitive arguments.
Some results however do exist, going back to~\citep{ShaweTaylor1991ThresholdNL}. ~\citet{MostafaHints} made the argument that restricting a classifier to be invariant can not increase its VC dimension. ~\citet{sokolic17a} extend this idea to derive generalization bounds for invariant classifiers, while ~\citet{Sannai2019ImprovedGB} do so specifically working with the permutation group. ~\citet{ElesedyGeneralization} show a strict generalization benefit for equivariant linear models, showing that the generalization gap between a least squares model and its equivariant version depends on the dimension of the space of anti-symmetric linear maps. Some benefits of related ideas such as data augmentation and invariant averaging are formally shown in~\citep{Lyle2020OnTB, ChenAugmentation}.
Here we focus on the limits to expressivity enforced by equivariance.

% Our work is distinct from work of this nature -- we instead study the expressivity of a representation that is constrained by group equivariance.

% of the number of trivial
% \newterm{irreducible} actions
% \newterm{irreps}
% (\emph{irreps})
% that appear in the decomposition of the group into a direct sum of irreducible
% actions \mc{is there a name for this decomposition I can use?}.
% With this theory we construct \mc{new?} convolutional neural network architectures that correspond with different group representations, and predictions of the capacity of these networks given by theory are tested in experiments with pre-trained network models.
% We find \mc{todo}.
% \mc{Does this generalize g-conv nets as considered so far or not?}

\section{Problem formulation}\label{sec:ProblemForm}
\vspace{-.1in}

%Suppose we consider an object $\vx$ represented as a vector $\vr(\vx) \in \sR^N$. We consider transformations of this object, such that they form a group in the algebraic sense of the word. Recall that a group $G$ is a set endowed with a composition operation $G \times G \to G$ satisfying certain axioms\footnote{axioms: $G$ contains an identity element $e$ such that $eg = ge = g$ for all $g$, for all $g \in G$ there is an inverse element $g^{-1}$ so that $g g^{-1} = g^{-1} g = e$, and  }. Groups are abstract mathematical objects. In order to make them concrete one way is to model their elements by invertible matrices. This is the key idea behind representation theory. 

Suppose $\vx$ abstractly represents an object and let $\vr(\vx) \in \mathbb{R}^N$ be some feature map of $\vx$ to an $N$-dimensional space (such as an intermediate layer of a deep neural network). 
We consider transformations of this object, such that they form a group in the algebraic sense of the word. %\footnote{A group $G$ is a set equipped with a binary operation $\cdot$ which satisfies the following axioms: \\
%a. Associativity: $g \cdot \left(h \cdot k \right) = \left( g \cdot h \right) \cdot k$ for all $g,h,k \in G$ \\
%b. Identity Element: $G$ contains an element $e$ such that $e \cdot g = g \cdot e = g$ for all $g \in G$ 
%\\
%c. Inverse Element: For all $g \in G$ there is an inverse element $g^{-1}$ so that $g \cdot g^{-1} = g^{-1} \cdot g = e$.
%\\
%Note that invertible linear maps on a vector space $V$ satisfy all three of these conditions. This group is known as $GL(V)$.} 
We denote the abstract transformation of $\vx$ by element $g \in G$ as $g\vx$. Groups $G$ may be represented by invertible matrices, which act on a vector space $V$ (which themselves form the group $GL(V)$ of invertible linear transformations on $V$). We are interested in feature maps $\vr$ which satisfy the following group equivariance condition:
\vspace{-.1in}
\begin{align*}
    \vr( g \vx) = \pi(g) \vr(\vx),
\end{align*}
where $\pi : G \to GL(\mathbb{R}^N)$ is a linear \newterm{representation} of $G$ which acts on feature map $\vr(\vx)$. Note that many representations of $G$ are possible, including the trivial representation: $\pi(g)= \bm I$ for all $g$.

We are interested in perceptual object manifolds generated by the actions of $G$.
Each of the $P$ manifolds can be written as a set of points $\{\pi(g) \vr ^ \mu: g\in G\}$ where $\mu \in [P] \equiv \{1,2,\dots,P\}$; that is, these manifolds are orbits of the point $\vr^{\mu} \equiv \vr (\vx^{\mu})$ under the action of $\pi$.
We will refer to such manifolds as \newterm{$\pi$-manifolds}.\footnote{
Note that for a finite group $G$, each manifold will consist of a finite set of points.
This violates the technical mathematical definition of ``manifold'', but we abuse the definition here for the sake of consistency with related work~\citep{chungClassificationGeometryGeneral2018,cohenGeneralTheoryEquivariant2019}; to be more mathematically precise, one could instead refer to these ``manifolds'' as $\pi$-orbits.}

Each of these $\pi$-manifolds represents a single object under the transformation encoded by $\pi$; hence, each of the points in a $\pi$-manifold is assigned the same class label.
A perceptron endowed with a set of linear readout weights $\vw$ will attempt to determine the correct class of every point in every manifold.
The condition for realizing (i.e. linearly separating) the dichotomy $\{y^{\mu}\}_{\mu}$ can be written as $y^{\mu}\vw^{\top}\pi(g)\vr^{\mu} > 0$ for all $g\in G$ and $\mu \in [P]$, where $y^{\mu} = +1$ if the $\mu$\textsuperscript{th} manifold belongs to the first class and $y^{\mu} = -1$ if the $\mu$\textsuperscript{th} manifold belongs to the second class.
The perceptron capacity is the fraction of dichotomies that can be linearly separated; that is, separated by a hyperplane.

For concreteness, one might imagine that each of the $\vr^\mu$ is the neural representation for an image of a dog (if $y^{\mu} = +1$) or of a cat (if $y^{\mu} = -1$).
The action $\pi(g)$ could, for instance, correspond to the image shifting to the left or right, where the size of the shift is given by $g$.
Different representations of even the same group can have different coding properties, an important point for investigating biological circuits and one that we leverage to construct a new GCNN architecture in \Secref{sec:gcnns}.

\section{Perceptron Capacity of Group-Generated Manifolds}\label{sec:perc-capacity}
\vspace{-.1in}
Here we first state and prove our results in the general case of representations of compact groups over $\sR^N$.
% \footnote{Semi-simplicity of the group is crucial for our results. Since we work with representations of compact $G$ defined over the real numbers, this is always the case.}
The reader is encouraged to read the proofs, as they are relatively simple and provide valuable intuition.
We consider applications to specific group representations and GCNNs in Sections \ref{sec:ex-cyclic} and \ref{sec:gcnns} that follow.

\subsection{Separability of \texorpdfstring{$\pi$}{π}-manifolds}\label{sec:cap-red}
\vspace{-.1in}
% Let $\pi$ be a linear representation of compact group $G$.

% This decomposition can be thought of as a simultaneous block-diagonalization of the representation $\pi$, i.e. $\pi(g) = \Phi \mB(g) \Phi^{*}$ where each $\mB(g)$ is a block diagonal matrix and $\Phi^{*}$ is the Hermitian transpose of the unitary matrix $\Phi$ (note that $B$ and $\Phi$ can have complex valued entries). \mc{Are we assured that each of the blocks in $\mB(g)$ have the same dimension across $g$?}.
% Each of the blocks of $\mB(g)$ in this decomposition corresponds with an irrep, and cannot be further decomposed into smaller blocks without losing information about $\pi$.

We begin with a lemma which states that classifying the $P$ $\pi$-manifolds can be reduced to the problem of classifying their $P$ centroids. For the rest of \Secref{sec:perc-capacity}, we let $\pi: G \to \gl (\sR^N)$ be an arbitrary linear representation of a compact group $G$.\footnote{Note that our results extend to more general vector spaces than $\sR^N$, under the condition that the group be semi-simple.}
We denote the average of $\pi$ over $G$ with respect to the Haar measure by $\avg{\pi(g)}_{g\in G}$; for finite $G$ this is simply $\frac{1}{|G|}\sum_{g\in G} \pi (g)$ where $|G|$ is the order (i.e. number of elements) of $G$.
For ease of notation we will generally write $\langle \pi \rangle \equiv \langle \pi (g) \rangle_{g\in G}$ when the group $G$ being averaged over is clear.
\begin{lemma}\label{thm:lemma-1}
A dataset $\{( \pi(g) \vr^\mu, y^\mu) \}_{g\in G, \mu \in [P]}$ consisting of $P$ $\pi$-manifolds with labels $y^{\mu}$ is linearly separable if and only if the dataset $\{( \avg{\pi} \vr^\mu, y^\mu) \}_{\mu \in [P]}$ consisting of the $P$ centroids $\avg{\pi} \vr^\mu$ with the same labels is linearly separable. Formally,
\begin{align*}
    \exists \vw \ \forall g \in G, \mu \in [P] :  y^\mu \vw^\top \pi(g) \vr^\mu  > 0 \iff \exists \vw \ \forall \mu \in [P] :   y^\mu \vw^\top \avg{\pi} \vr^\mu > 0.
\end{align*}
\vspace{-.35in}
\end{lemma}
\begin{proof}
% Using the fact that the representation $\pi$ is unitary, we can reduce this problem of classifying $P$ orbits to the problem of separating $P$ points, namely the $G$-averages of the orbits. Formally, we establish the logical equivalence between the $G$-invariant storage problem and the classical storage problem on the dataset $\mathcal{D}' = \{ (\left< \pi(g) \right>_{g\in G} r^\mu ,y^\mu ) \}$

% \begin{align}
%     \exists w :  y^\mu w^\top \pi(g) r^\mu  > 0 \iff \exists w : y^\mu w^\top \left< \pi(g) \right>_{g \in G} r^\mu > 0
% \end{align}
% Let $\mathcal{D}' = \{ (\left< \pi(g) \right>_{g\in G} r^\mu ,y^\mu ) \}$.
The forward implication is obvious: if there exists a $\vw$ which linearly separates the $P$ manifolds according to an assignment of labels $y^{\mu}$, that same $\vw$ must necessarily separate the centroids of these manifolds. This can be seen by averaging each of the quantities $y^\mu \vw^\top \pi(g) \vr^\mu$ over $g \in G$. Since each of these quantities is positive, the average must be positive.
% This result can also be shown at once by observing that the centroid of a set of points is always contained in their convex hull.

For the reverse implication, suppose $y^{\mu}\vw^{\top} \langle \pi \rangle \vr^{\mu}>0$, and define $\tilde{\vw} = \langle \pi \rangle^{\top} \vw$. We will show that $\tilde{\vw}$ separates the $P$ $\pi$-manfolds since 
\begin{align*}
    y^{\mu} \tilde{\vw}^{\top} \pi (g)\vr^{\mu}&=y^{\mu} \vw^{\top}\langle \pi \rangle \pi (g)\vr^{\mu} \nonumber \ \  \text{(Definition of $\tilde{\vw}$)} 
\\
&=y^{\mu} \vw^{\top}\langle \pi(g') \pi (g)\rangle_{g'\in G}\vr^{\mu} \nonumber \ \  \text{(Definition of $\left< \pi \right>$ and linearity of $\pi (g)$)} 
% \\
% &= y^{\mu} \vw^{\top}\langle \pi(g'g) \rangle_{g'\in G}\vr^{\mu}\nonumber \ \  \text{($\pi(g)$ is a representation of $G$)} 
\\
&= y^{\mu} \vw^{\top}\langle \pi \rangle\vr^{\mu} \ \  \text{(Invariance of the Haar Measure $\mu(S g) = \mu(S)$ for set $S$)}  \nonumber
\\
&>0 \ \ \ \  \text{(Assumption that $\vw$ separates centroids)}
\end{align*}
Thus, all that is required to show that $\tilde{\vw}$ separates the $\pi$-orbits are basic properties of a group representation 
% $\pi(g')\pi(g) = \pi(g'g)$
and invariance of the Haar measure to $G$-transformations.
%We see that all that is required to show that $\tilde{\vw}$ separates the $P$ $\pi$-manifolds are basic properties of a group representation $\pi(g')\pi(g) = \pi(g'g)$ and the invariance of the measure for $G$-transformations. Using these facts, we showed that $\tilde{\vw}$ separates the $\pi$-orbits.
\end{proof}

\subsection{Relationship with Cover's Theorem and VC dimension}
\vspace{-.1in}
The fraction $f$ of linearly separable dichotomies on a dataset of size $P$ for datapoints in general position
\footnote{A set of $P$ points is in \newterm{general position} in $N$-space if every subset of $N$ or fewer points is linearly independent.
This says that the points are ``generic'' in the sense that there aren't any prescribed special linear relationships between them beyond lying in an $N$-dimensional space. Points drawn from a Gaussian distribution with full-rank covariance are in general position with probability one.}
in $N$ dimensions was computed
% by~\citet{Wendel1962API} for the case when the points reside on a unit sphere,
% and generalized to arbitrary feature maps
by~\citet{Cover1965GeometricalAS}
and takes the form
\begin{equation}\label{CoverThm}
    f(P,N) = 2^{1-P} \sum_{k=0}^{N-1} 
    \binom{P-1}{k}
    % {P-1 \choose k}.
\end{equation}
where we take
% we use the convention that
$\binom{n}{m}=0$ for $m>n$.
Taking $\alpha = P/N$, $f$ is a nonincreasing sigmoidal function of $\alpha$ that takes the value 1 when $\alpha \leq 1$, 1/2 when $\alpha=2$, and approaches $0$ as $\alpha \to \infty$.
In the thermodynamic limit of $P, N \to \infty$ and $\alpha=O(1)$, $f(\alpha)$ converges to a step function
% \begin{equation}
$
\label{GardnerThm}
\bar{f}(\alpha) = \lim_{P,N \to \infty, \alpha = P/N} f(P,N) =\Theta(2-\alpha)
$~\citep{Gardner1987MaximumSC,Gardner1988TheSO,shcherbina_rigorous_2003},
% \end{equation}
indicating that in this limit all dichotomies are realized if $\alpha < 2$ and no dichotomies can be realized if $\alpha>2$, making $\alpha_c = 2$ a critical point in this limit.
 
The VC dimension~\citep{Vapnik1971ChervonenkisOT} of the perceptron trained on points in dimension $N$ is defined to be the largest number of points $P$ such that all dichotomies are linearly separable for some choice of the $P$ points. These points can be taken to be in general position.\footnote{This is because linear dependencies decrease the number of separable dichotomies (see \cite{hertz2018introduction}).}
Because $f(P,N)=1$ precisely when $P \leq N$, the VC dimension is always $N=P$ for any finite $N$ and $P$ (see \cite{abu2012learning}), even while the asymptotics of $f$ reveal that $f(P,N)$ approaches $1$ at any $N>P/2$ as $N$ becomes large.
In this sense the perceptron capacity yields strictly more information than the VC dimension, and becomes comparatively more descriptive of the expressivity as $N$ becomes large.
%This comparison between VC dimension and perceptron capacity is to the best of our knowledge a novel observation.

In our $G$-invariant storage problem, the relevant scale is $\alpha = P / N_0$ where $N_0$ is the dimension of the fixed point subspace (defined below).
%  $N_0$ is the number of trivial irreducible dimensions (defined below).
%  For $\alpha > 2$, we expect no dichotomies will be learned.
% The proof of our main theorem is a direct application of \Lemref{thm:lemma-1}.
We are now ready to state and prove our main theorem.

\begin{theorem}\label{thm:thm-1}
Suppose the points $\avg{\pi} \vr^{\mu}$ for $\mu \in [P]$ lie in general position in the subspace $V_0 = \range (\avg{\pi}) = \{\avg{\pi}\vx : \vx \in V\}$.
% For a compact group $G$ with representation $\pi$
% over the reals\footnote{more generally, the representation must be semi-simple}
% such that $\left< \pi(g) \right>_{g \in G} \vr^{\mu}$ lie in general position in a subspace $V_0$,
Then $V_0$ is the fixed point subspace of $\pi$, and
the fraction of linearly separable dichotomies on the $P$ $\pi$-manifolds $\{\pi (g) \vr^{\mu}: g \in G\}$ is $f(P,N_0)$,
where $N_0 = \dim V_0$.
% is the dimension of the fixed point subspace of $\pi$.
Equivalently, $N_0$ is the number of trivial irreducible representations that appear in the decomposition of $\pi$ into irreducible representations.
\end{theorem}
% ST: feel free to use the old statement, just re-worded it above, keeping the older version below. 
%\begin{theorem}\label{thm:thm-1}
%The fraction of linearly separable dichotomies on the $P$ $\pi$-manifolds $\{\pi (g) \vr^{\mu}: g \in G\}$ with $\mu \in \{1,2,...,P\}$ for unitary representation $\pi$ of compact semi-simple group $G$ and $\left< \pi(g) \right>_{g \in G} \vr^{\mu}$ lie in general position in $V_0$ is $f(P,N_0)$,
%where $N_0$ is the dimension of the invariant subspace of $\pi$.
%Equivalently, $N_0$ is the number of trivial irreps that appear in the decomposition of $\pi$ into irreps.
%\end{theorem}
% \bb{Since we got rid of the direct sum definition and the orthogonality in the lemma proof, we need to add it back here. Ultimately we want to show that $\left< \pi \right>$ is a projection matrix to trivial subspace so fraction of separable dichotomies is $f(P,N_0)$.}
The \newterm{fixed point subspace} is the subspace $W=\{w \in V | gw=w \text{ ,} \forall g \in G\}$.
The theorem and its proof use the notion of \newterm{irreducible} representations (irreps), which are in essence the fundamental building blocks for general representations.
Concretely, any representation of a compact group over a real vector space decomposes into a direct sum of irreps~\citep{NaimarkStern, BumpLie}.
A representation $\pi : G \to \gl(V)$ is irreducible if $V$ is not 0 and if no vector subspace of $V$ is stable under $G$, other than $0$ and $V$ (which are always stable under $G$).
A subspace $W$ being stable (or invariant) under $G$ means that $\pi (g)\vw \in W$ for all $\vw \in W$ and $g \in G$.
The condition that the points $\avg{\pi} \vr^{\mu}$ be in general position essentially means that there is no prescribed special relationship between the $\vr^{\mu}$ and between the $\vr^{\mu}$ and $\avg{\pi}$.
Taking the $\vr^{\mu}$ to be drawn from a full-rank Gaussian distribution is sufficient to satisfy this condition.
% In this work we consider the case where $V$ is a vector space over the real numbers since we are interested in real-valued representations.
\vspace{-.1in}
\begin{proof}
By the theorem of complete reducibility (see~\cite{FultonHarris}), $\pi$ admits a decomposition into a direct sum of irreps $\pi \cong \pi_{k_1} \oplus \pi_{k_2} \oplus ... \oplus \pi_{k_M}$ acting on vector space $V = V_1 \oplus V_2 \oplus ... V_M$,
where $\cong$ denotes equality up to similarity transformation.
The indices $k_j$ indicate the type of irrep corresponding to invariant subspace $V_j$.
The fixed point subspace $V_0$ is the direct sum of subspaces where trivial $k_j=0$ irreps act: $V_0 = \bigoplus_{n: k_n = 0} V_n$.
% By the Grand Orthogonality Theorem~\citep{LiboffGOT}
By the Grand Orthogonality Theorem of irreps (see~\cite{LiboffGOT})
% Through the orthogonality of irreducible representations of compact groups~\citep{LiboffGOT}
% (by virtue of the Grand Orthogonality Theorem~\citep{LiboffGOT}),
% of compact groups,
all non-trivial irreps average to zero $\left< \pi_{k,ij}(g) \right>_{g\in G} \propto \delta_{k,0}\delta_{i,j}$.
Then, the matrix $\avg{\pi}$ simply projects the data to $V_0$.
% Denoting $N_0 = \text{dim}  V_0$, 
By \Lemref{thm:lemma-1} the fraction of separable dichotomies on the $\pi$-manifolds is the same as that of their centroids $\avg{\pi} \vr^{\mu}$.
% , which lie in the space $V_0$.
Since, by assumption, the $P$ points $\avg{\pi} \vr^{\mu}$ lie in general position in $V_0$, the fraction of separable dichotomies is $f(P, \dim V_0)$ by \Eqref{CoverThm}.
% \citep{ Wendel1962API, Cover1965GeometricalAS}.
\end{proof}

\textbf{Remark:} A main idea used in the proof is that only nontrivial irreps average to zero.
This will be illustrated with examples in \Secref{sec:ex-cyclic} below.

\textbf{Remark:} For finite groups, $N_0$ can be easily computed by averaging the trace of $\pi$, also known as the \textit{character}, over $G$: $N_0 = \avg{ \Tr (\pi (g)) }_{g\in G} = \Tr (\avg{\pi})$ (see~\cite{serreLinearRepresentationsFinite2014}).
% \mc{I'm only aware of a reference for finite groups -- Shubhendu maybe you can cover this?} \st{I will have to look for a ref. To integrate for infinite groups you will need a measure on the characters, I am not sure what it is. I presume it will be something like Plancherel, but will check.} \cep{We can make a restricted claim for finite groups.}

\textbf{Remark:} If the perceptron readout has a bias term $b$ i.e. the output of the perceptron is $\vw^{\top}\pi (g) \vr + b$, then this can be thought of as adding an invariant dimension. This is because the output can be written $\tilde{\vw}^{\top}\tilde{\pi} (g) \tilde{\vr}$ where $\tilde{\vw} = (\vw, b)$, $\tilde{\vr}=(\vr, 1)$, and $\tilde{\pi}(g) = \pi(g) \oplus 1$ is $\pi$ with an extra row and column added with a 1 in the last position and zeros everywhere else.
Hence the fraction of separable dichotomies is $f(P, N_0+1)$.

These results extend immediately to a notion of VC-dimension of group-invariant perceptrons.
\begin{corollary}\label{thm:corr-3}
Let the VC dimension for $G$-invariant perceptrons with representation $\pi$, $N_{VC}^{\pi}$, denote the maximum number $P$, such that there exist $P$ anchor points $\{ \vr^\mu \}_{\mu=1}^P$ so that $\{ (\pi(g) \vr^\mu, y^{\mu}) \}_{\mu \in [P], g \in G}$ is separable for all possible binary dichotomies $\{y^\mu\}_{\mu\in [P]}$. Then $N_{VC}^{\pi} = \text{dim}(V_0)$.  
\end{corollary}
% \vspace{-.08in}
\begin{proof}
    By theorem \ref{thm:thm-1} and Equation \ref{CoverThm}, all possible dichotomies are realizable for $P \leq \text{dim}(V_0)$ provided the points $\left< \pi \right> \vr^\mu$ are in general position. 
\end{proof}
% \vspace{-.1in}
% \subsection{Capacity for Subgroups}

Our theory also extends immediately to subgroups. For example, strided convolutions are equivariant to subgroups of the regular representation.
\begin{corollary}\label{thm:corr-2}
A solution to the the $G$-invariant classification problem necessarily solves the $G'$-invariant problem where $G'$ is a subgroup of $G$. 
\end{corollary}
% \vspace{-.08in}
\begin{proof}
Assume that $y^\mu \vw^\top \vr(g \vx^\mu) > 0$ for $g \in G$. $G' \subseteq G$ $\implies$ $y^\mu \vw^\top \vr(g \vx^\mu) > 0$ $\forall g \in G'$. 
\end{proof}
\vspace{-.1in}
Consequently, the $G'$-invariant capacity is always higher than the capacity for the $G$-invariant classification problem.

% Finally, we state the asymptotic capacity of a group-invariant perceptron by applying \Eqref{GardnerThm}
% % Gardner's result \citep{Gardner1987MaximumSC,Gardner1988TheSO}:
% %
% \begin{corollary}
% The $G$-invariant perceptron capacity for a $\pi$-equivariant code in the asymptotic limit $N_0, P \to \infty$ with $\alpha = P/N_0 = O(1)$, is $\alpha_c = 2$.%, indicating that the maximum number of linearly separable $\pi$ orbits $P_c$ is $2 N_0$.  
% \end{corollary}

%\subsection*{Relation to the VC Dimension}
%\cep{Expand or remove}

\section{Example Application: The cyclic group \texorpdfstring{$\cyc{m}$}{Cm}}\label{sec:ex-cyclic}
\vspace{-.1in}
\begin{figure}
    \centering
    \includegraphics[width=\linewidth]{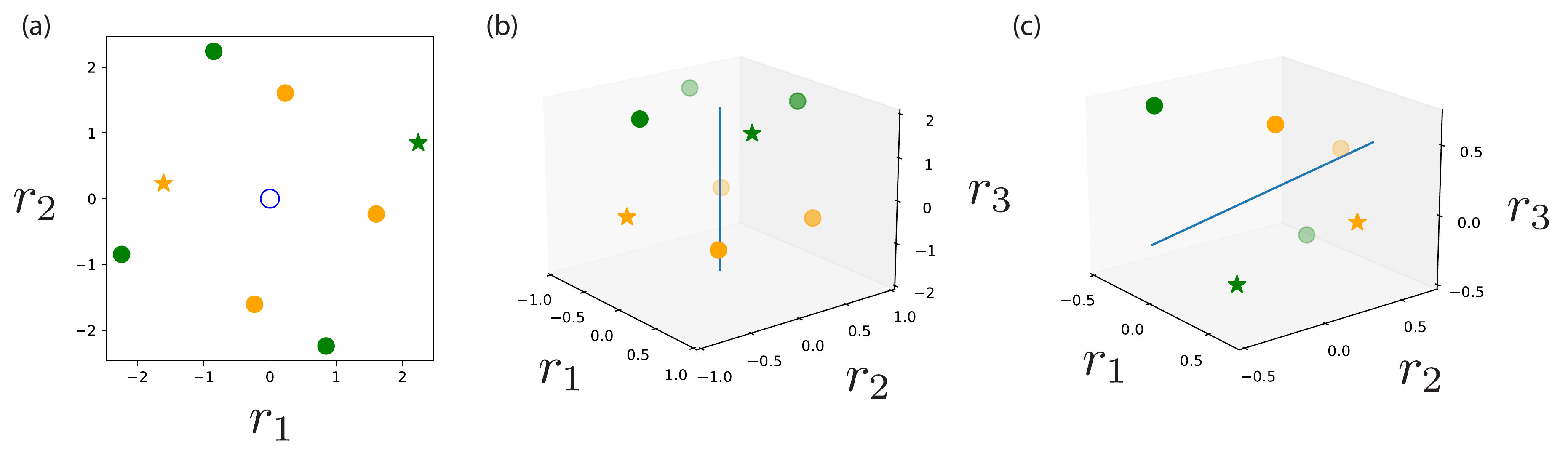}
    \caption{$\pi$-manifolds for different $\pi$, illustrating that only the fixed point subspace contributes to capacity. In each panel two manifolds are plotted, with color denoting class label. Stars indicate the random points $\vr^{\mu}$ for $\mu \in \{1,2\}$ where the orbits begin, and closed circles denote the other points in the $\pi$-manifold. For (a) and (b) the group being represented is $G=\cyc{4}$ and for (c) $G=\cyc{3}$. (a) Here $\pi(g)$ is the $2 \times 2$ rotation matrix $\mR(2 \pi g / 4)$. The open blue circle denotes the fixed point subspace $\{{\bm 0} \}$. (b) Here $\pi(g)$ is the $3\times3$ block-diagonal matrix with the first $2\times2$ block being $\mR(2 \pi g / 4)$ and second $1\times1$ block being $1$. The blue line denotes the fixed point subspace $\vspan \{(0,0,1)\}$. (c) Here $\pi(g)$ is the $3 \times 3$ matrix that cyclically shifts entries of length-$3$ vectors by $g$ places. The blue line denotes the fixed point subspace $\vspan \{(1,1,1)\}$. } 
    \label{fig:fig1}
\end{figure}
In this section we illustrate the theory in the case of the cyclic group $G = \cyc{m}$ on $m$ elements.
% with example representations of $\cyc{m}$.
This group is isomorphic to the group of integers $\{0,1,...,m-1\}$ under addition modulo $m$, and this is the form of the group that we will consider.
An example of this group acting on an object is an image that is shifted pixel-wise to the left and right, with periodic boundaries.
In Appendix \ref{app:so3-example} we show an application of our theory to the non-abelian Lie group $SO(3)$.
% A key insight is that $G$ can have many different representations $\pi : G \to \gl (V)$.

% -- in this case $V = \sR^2$ and $\pi(g) = \mR(\theta_g)$.
% Another valid representation is the \newterm{regular representation}, which for the cyclic group are the cyclic shift permutation matrices -- $V=\sR^m$ and $\pi(g)$ is the matrix that shifts the coordinates of length $m$ vectors $g$ places (with periodic boundary conditions).
% Another valid representation can be realized by defining a representation on the isomorphic group
% $\cyc{m_1} \oplus \cyc{m_2} \cong \cyc{m}$ for $m_1$ and $m_2$ coprime such that $m = m_1 m_2$. This representation will be treated at the end of the section.

% \subsubsection{Real Irreducible Representations of the Cyclic Group}
\subsection{Rotation matrices}
\vspace{-.1in}
% \bb{All the real irreps of $\mathbb{Z}^n$ have the form of either 2D discrete rotation matrices like $\bm R$ or 1 dimensional scalars like $\rho_0(g_k) = 1, \rho_1(g_k) = (-1)^k$. An interesting example where both irreps are one-dimensional and real is the regular representation of $\mathbb{Z}^2$, where the span of $(1,1)$ gives $\rho_{(1,1)}(g_k) = 1$ while the group action on the span $(1,-1)$ gives $\rho_{(1,-1)}(g_k) = (-1)^k$. I think if we want to argue that counting the zero-frequency modes is the same as counting dimension of the trivial subspace, then we may want to make this connection clear. }
% \mc{Fair enough, but so far I think that I am only talking about the lowest-frequency nontrivial irrep in this section. It's possible that expanding the scope of the section could be good.}
The $2\times 2$ discrete rotation matrices
$\mR(\theta_g) \equiv \begin{bmatrix} \cos(\theta_g) &  -\sin(\theta_g) \\
    \sin(\theta_g) & \cos(\theta_g)\end{bmatrix}$
where $\theta_g = 2 \pi g / m$ and $g\in \cyc{m}$, are one possible representation of $\cyc{m}$; in this case $V = \sR^2$.
This representation is irreducible and nontrivial, which implies that the dimension of the fixed point subspace is $0$ (only the origin is mapped to itself by $\mR$ for all $g$).
% , meaning that there is no nontrivial proper subspace of $V$ stable under the action of $\mR$.
% Suppose $\pi: G \to \gl (V)$ is the discrete rotation representation of the cyclic group, $\pi(k) = \mR(\theta_k)$
% for $k\in \{0,1,...,m-1\}$ and $\theta_k = 2 \pi k / m$, and where $\mR$ is the rotation matrix as defined above.
% Then $\pi$ is an irrep, since no $1$-d subspaces of $V$ are fixed by $\mR$.
% This representation is nontrivial provided that $k > 0$.
Hence the fraction of linearly separable dichotomies of the $\pi$-manifolds by \Thmref{thm:thm-1} is $f(P,0)$.
% In particular, the capacity is independent of $m$ for $m>1$.
This result can be intuitively seen by plotting the orbits, as in \Figref{fig:fig1}a for $m=4$.
Here it is apparent that it is impossible to linearly separate two or more manifolds with different class labels, and that the nontrivial irrep $\mR$ averages to the zero matrix.

% The proof of \Lemref{thm:thm-1} uses the fact that the nontrivial irreps average to the zero matrix $\boldsymbol{0}$.
% This can be seen in this example directly in this special case by averaging $\mR(\theta_k)$ over $k \in G$, and is visually apparent in \Figref{fig:fig1}a.

The representation can be augmented by appending trivial irreps, defining $\pi: G \to \gl (\sR ^ {N})$ by $\pi (g) = \mR(\theta_g) \oplus \mI \equiv \begin{bmatrix} \mR(\theta_g) &  0 \\
    0 & \mI \end{bmatrix}$
where $\mI$ is an $(N-2) \times (N-2)$-dimensional identity matrix.
The number of trivial irreps is $N-2$, so that the capacity is $f(P, N-2)$.
This is illustrated in \Figref{fig:fig1}b for the case $N=3$.
Here we can also see that the trivial irrep, which acts on the subspace $\vspan \{ (0,0,1)\}$, is the only irrep in the decomposition of $\pi$ that does not average to zero.
This figure also makes intuitive the result of \Lemref{thm:lemma-1} that dichotomies are realizable on the $\pi$-manifolds if and only if the dichotomies are realizable on the centroids of the manifolds.

%Finally, note that in the limit $m\to \infty$ the matrix group $\{\mR(\theta_k): k \in G\}$ becomes the continuous Lie group $\text{SO}(2)$.
%In this case there are infinitely many irreps in the decomposition of $\pi$, but still only $n-2$ copies of the trivial irrep.
%This is illustrated in \Figref{fig:fig1} in the case $n=3$, and where $\pi$ has been rotated by a unitary matrix (this rotates the subspaces for the irreps but otherwise does not affect the decomposition).
% \mc{need to deal with the fact that Figure 1 is a rotated version}.

\subsection{The regular representation of \texorpdfstring{$\cyc{m}$}{Cm}}
\vspace{-.1in}
Suppose $\pi: \cyc{m} \to \gl (V)$ is the representation of $\cyc{m}$ consisting of the cyclic shift permutation matrices (this is called the \newterm{regular representation} of $\cyc{m}$).
In this case $V=\sR^{m}$ and $\pi (g)$ is the matrix that cyclically shifts the entries of a length-$m$ vector $g$ places.
For instance, if $m=3$ and $\vv=(1,2,3)$ then $\pi (2) \vv = (2, 3, 1)$.

% The nontrivial 2-dimensional irreps have the form
% $\rho_j (g) = \mR(j \theta_g)
%\begin{bmatrix} \cos(j\theta_k) &  -\sin(j\theta_k) \\
%\sin(j\theta_k) & \cos(j\theta_k)\end{bmatrix}
% $
% where $\theta_k = 2 \pi k / m$ for $k\in \{0,1,...,m-1\}$ and $j \in \{1,...,m-1\}$ (see Appendix \ref{app:cyclic-irreps} for a derivation).
% When $m$ is even then there is a single one-dimensional nontrivial irrep, which captures the effect of rotating by 180 degrees: $\rho_{m/2}(k) = (-1)^k$.
In Appendix \ref{app:cyclic-irreps} we derive the irreps of this representation,
which consist of rotation matrices of different frequencies.
There is one copy of the trivial irrep $\pi_0 (g) \equiv 1$ corresponding with the fixed point subspace $\vspan\{\vone_m\}$ where $\vone_m$ is the length-$m$ all-ones vector. Hence the fraction of separable dichotomies is $f(P,1)$.
This is illustrated in \Figref{fig:fig1}c in the case where $m=3$.
% Here it is apparent that the average of the orbits lies in the subspace
% acted on by the trivial irrep, which is
% $\vspan \{ (1,1,1)\}$.
% Note that this is equivalently the subspace invariant to the action of $\pi$.
The average
% representation matrix for 
of the regular representation matrix is $\avg{\pi} = \frac{1}{|G|} \vone_m \vone_m^\top$, 
% which is rank-one.
indicating that $\avg{\pi}$ projects
% the data \vr^{\mu}
data along $\vone_m$.
% the $\vone_m$ direction.
% For example, for $|G| = 3$, we have
% \begin{align}
%     \avg{\pi} = \frac{1}{3} \begin{bmatrix} 1 & 0 & 0 \\ 0 & 1 & 0 \\ 0 & 0 & 1 \end{bmatrix} + \frac{1}{3} \begin{bmatrix} 0 & 1 & 0 \\ 0 & 0 & 1 \\ 1 & 0 & 0 \end{bmatrix} + \frac{1}{3} \begin{bmatrix} 0 & 0 & 1 \\ 1 & 0 & 0 \\ 0 & 1 & 0 \end{bmatrix} = \frac{1}{3} \begin{bmatrix} 1 & 1 & 1 \\ 1 & 1 & 1 \\ 1 & 1 & 1 \end{bmatrix} .
% \end{align}
% Thus, a linear classifier on a regular representation of $\cyc{N}$ will only separate $f(P,1)$ fraction of dichotomies.

\subsection{Direct sums of Regular Representations}\label{sec:ex-direct-sum}
\vspace{-.1in}
For our last example we define a representation using the isomorphism $\cyc{m} \cong \cyc{m_1} \oplus \cyc{m_2}$ for $m = m_1 m_2$ and $m_1$ and $m_2$ coprime\footnote{Two numbers are coprime if they have no common prime factor.}.
Let
$\pi^{(1)} : \cyc{m_1} \to \gl (\sR ^ {m_1})$ and $\pi^{(2)} : \cyc{m_2} \to \gl (\sR ^ {m_2})$
be the cyclic shift representations (i.e. the regular representations) of $\cyc{m_1}$ and $\cyc{m_2}$, respectively.
Consider the representation $\pi^{(1)} \oplus \pi^{(2)}: \cyc{m} \to \gl (\sR^{m_1+m_2}) $ defined by $(\pi^{(1)} \oplus \pi^{(2)})(g) \equiv \pi^{(1)}(g \mod m_1) \oplus \pi^{(2)}(g \mod m_2)$,
% where
% $\pi^{(1)}(g \mod m_1) \oplus \pi^{(2)}(g \mod m_2)$, is 
the block-diagonal matrix with $\pi^{(1)}(g \mod m_1)$ being the first and $\pi^{(2)}(g \mod m_2)$ the second block.
% Note that $\pi^{(1)}(g \mod m_1)$ and $\pi^{(2)}(g \mod m_2)$ are regular representations of $\cyc{m_1}$ and $\cyc{m_2}$, respectively.
% This is in fact the same matrix group as the one gotten by taking all matrices $\pi^{(1)}(k) \oplus \pi^{(2)}(k')$ where $k \in \cyc{m_1}$ and $k' \in \cyc{m_2}$, a consequence of the isomorphism $\cyc{m} \cong \cyc{m_1} \oplus \cyc{m_2}$
% , since an element of $\cyc{m}$ can be uniquely decomposed into a sum of elements of $\cyc{m_1}$ and $\cyc{m_2}$ \mc{check}.

% The nontrivial irreps of $\pi^{(1)} \oplus \pi^{(2)}$ have the form  $\pi_j^{(1)}(g) = \mR(2 \pi j  g/m_1)$ and $\pi_{j'}^{(2)}(g) = \mR(2 \pi j'  g/m_2)$ for $j \in \{1,2,...,m_1-1\}$ and $j' \in \{1,2,...,m_2-1\}$.
There are two copies of the trivial representation in the decomposition of $\pi^{(1)} \oplus \pi^{(2)}$, corresponding to the one-dimensional subspaces $\vspan \{ (\vone_{m_1}, \vzero_{m_2})\}$ and $\vspan \{(\vzero_{m_1}, \vone_{m_2})\}$, where $\vzero_{m_1}$ is the length-$k$ vector of all zeros.
% , where $\vone_k$ and $\vzero_k$ are the length-$k$ vectors of all ones and zeros, respectively.
% Note that $\vone_{m_1 + m_2}$ is also invariant to the action of $\pi$, but lies in the span of $\{(\vone_{m_1}, \vzero_{m_2}),(\vzero_{m_1}, \vone_{m_2})\}$ and so doesn't contribute to the invariant subspace dimension.
Hence the fraction of separable dichotomies is $f(P,2)$.
% These two copies of the trivial irrep each contribute a dimension to the calculation of capacity, yielding a capacity of $f(P,2)$.
This reasoning extends simply to direct sums of arbitrary length $\ell$, yielding a fraction of $f(P,\ell).$\footnote{Our results do not actually require that the $m_k$ be coprime, but rather that none of the $m_k$ divide one of the others.
To see this, take $\tilde{m}$ and $\tilde{m}_k$, to be $m$ and the $m_k$ after being divided by all divisors common among them.
Then $\cyc{\tilde{m}}\cong \oplus_{k=1}^{\ell} \cyc{\tilde{m}_k}$ and
provided none of the $\tilde{m}_k$ are $1$, one still gets a fraction of $f(P, \ell)$.
% the fraction of $\Oplus_k \pi_k$
}
These representations are used to build and test a novel G-convolutional layer architecture with higher capacity than standard CNN layers in Section \ref{sec:gcnns}. 

These representations are analogous to certain formulations of grid cell representations as found in entorhinal cortex of rats~\citep{haftingMicrostructureSpatialMap2005}, which have desirable qualities in comparison to place cell representations~\citep{sreenivasanGridCellsGenerate2011}.\footnote{Place cells are analogous to standard convolutional layers.}
Precisely, a collection $\{\cyc{m_k} \times \cyc{m_k}\}_k$ of grid cell modules encodes a large 2-dimensional spatial domain $\cyc{m}\times \cyc{m}$ where $m=\prod_k m_k$.
% Each grid cell module can be thought of as \cep{Cryptic: expand or remove}

% and $0 \oplus R(2 \pi j k)$ for $j \in {1,2,...,m}$ 
% Given $k \in \sZ / m\sZ$, $k$ can be written as $k = (k \mod m_1 + k \mod m_2)  \mod m$.
% Furthermore

% : $\langle \pi(g)  \rangle_g \pi(g') \vr^{\mu} = \langle \pi(g) \rangle_g  \pi(g'')\vr^{\mu}$ for every $g', g'' \in G$.

\section{G-Equivariant Neural Networks}\label{sec:gcnns}
\vspace{-.1in}

The proposed theory can shed light on the feature spaces induced by $G$-CNNs.
% As a concrete example,
Consider a single convolutional layer feature map for a finite\footnote{We consider finite groups here for simplicity, but the theory extends to compact Lie groups.} group $G$ with the following activation
\begin{align}
    a_{i,k}(\vx) = \phi( \vw_i^{\top} g_k^{-1} \vx ) \ , \  g_k \in G \ , \ i \in \{ 1,...,N \}
\end{align}
for some nonlinear function $\phi$.
For each filter $i$, and under certain choices of $\pi$, the feature map $\va_i(\vx) \in \mathbb{R}^{|G|}$ exhibits the equivariance property $\va_i(g_k \vx) = \pi(g_k) \va_i(\vx)$.
We will let $\bm a(x) \in \mathbb{R}^{|G| N}$ denote a flattened vector for this feature map.

% For example, 
In a traditional periodic convolutional layer applied to inputs of width $W$ and length $L$, the feature map of a single filter $\va_i(\vx) \in \mathbb{R}^{|G|}$ is equivariant with the regular representation of the group $G = \cyc{W} \times \cyc{L}$ (the representation that cyclically shifts the entries of $W\times L$ matrices).
% For images of size $W\times H$,
Here the order of the group is $|G| = WL$.
Crucially, our theory shows that this representation contributes exactly one trivial irrep per filter (see Appendix \ref{app:standard-cnn}).
Since the dimension of the entire collection of $N$ feature maps $\va(x)$ is $N|G|$ for finite groups $G$, one might naively expect capacity to be $P \sim 2 N |G|$ for large $N$ from \citet{Gardner1987MaximumSC}.
However, Theorem \ref{thm:thm-1} shows that for $G$-invariant classification, only the trivial irreps contribute to the classifier capacity.
Since the number of trivial irreps present in the representation is equal to the number of filters $N$, we have $P \sim 2N$.
% Further, if we produce a new layer $\bm r(x)$ through the global average pooling operation 
% \begin{equation}
%     r_i(x) = \frac{1}{|G|}\sum_{k=1}^{|G|} a_{i,k}(x) \ , \ i \in \{1,...,N\}
% \end{equation}
% the new code $\bm r(x) \in \mathbb{R}^N$ is $G$-invariant: $\bm r(g x) = \bm r(x)$ for all $g\in G$ (ie possesses the trivial representation of $G$). Our theory predicts that the hypothesis class $\mathcal H_{\bm a}$ for readouts from $\bm a(x,g)$ and the hypothesis class $\mathcal H_{\bm r}$ of readouts from the global average pooling layer $\bm r(x)$, defined as
% \begin{align}
%     \mathcal H_{\bm a} = \left\{ \vw \in \mathbb{R}^{|G|N} \to y(x, \vw) = \text{sign}\left( \sum_{ik} w_{ik} a_{i,k}(x) \right)\right \} \ , \ \mathcal H_{\bm r} = \left\{ \vw \in \mathbb{R}^N \to y(x, \vw) = \text{sign}\left( \sum_{i} w_{i} r_{i}(x) \right)\right \}
% \end{align}
% will have identical capacity on the $G$-invariant classification problem, despite the model family $\mathcal H_{\bm a}$ having a greater number of trainable parameters.

We show in Figure \ref{fig:capacity_plots} that our prediction for
% predicted fractionof linearly separable dichotomies
$f(P,N)$ matches
% the empirically measured fraction obtained
that empirically measured by training logistic regression linear classifiers on the representation. We perform this experiment on both (a) a random convolutional network and (b) VGG-11~\citep{Simonyan2015VeryDC} pretrained on CIFAR-10~\citep{krizhevskyLearningMultipleLayers2009} by \citep{kuangliu2021}.
% The pretrained models are taken from~\citep{kuangliu2021}.
For these models we vary $\alpha$ by fixing the number of input samples and varying the number of output channels by simply removing
% a varying number of
channels from the output tensor.
The convolutions in these networks are modified to have periodic boundary conditions while keeping the actual filters the same -- see Appendix \ref{app:standard-cnn} and \Figref{fig:non-periodic-conv} for more information and the result of using non-periodic convolutions, which impact the capacity but not the overall scaling with $N_0$.

Other GCNN architectures can have different capacities.
For instance, a convolutional layer equivariant to the direct sum representation of \Secref{sec:ex-direct-sum} has double the capacity with $P \sim 4N$ (\Figref{fig:capacity_plots}c), since each output channel contributes two trivial irreps.
See Appendix \ref{app:direct-sum-cnn} for an explicit construction of such a convolutional layer and a derivation of the capacity.
% \cep{we should add somewhere that we consider periodic b.c.s}

\subsection{Pooling Operations}
\vspace{-.1in}

In convolutional networks, local pooling is typically applied to the feature maps which result from each convolution layer.
In this section, we describe how our theory can be adapted for codes which contain such pooling operations.
We will first assume that $\pi$ is an $N$-dimensional representation of $G$. Let $\mathcal P(\bm r) : \mathbb{R}^{N} \to \mathbb{R}^{N/k}$ be a pooling operation which reduces the dimension of the feature map. The condition that a given dichotomy $\{y^\mu\}$ is linearly separable on a pooled code is 
\begin{align}
    \exists \vw \in \mathbb{R}^{N/k} \  \forall \mu \in [P] , g\in G : y^\mu \vw^\top \mathcal P( \pi(g) \bm r^\mu ) > 0
\end{align}
We will first analyze the capacity when $\mathcal P(\cdot)$ is a linear function (average pooling) before studying the more general case of local non-linear pooling on one and two-dimensional signals.

\subsubsection{Local Average Pooling}\label{sec:local_avg_pooling}
In the case of average pooling, the pooling function $\mathcal P(\cdot)$ is a linear map, represented with matrix $\mP$ which averages a collection of feature maps over local windows. Using an argument similar to Lemma \ref{thm:lemma-1} (Lemma \ref{thm:lemma-A1} and Theorem \ref{app:thm_avg_pool_capacity} in Appendix \ref{app:pooling}), we prove that the capacity of a standard CNN is not changed by local average pooling: for a network with $N$ filters, local average pooling preserves the one trivial dimension for each of the $N$ filters. Consequently the fraction of separable dichotomies is $f(P,N)$.

\subsubsection{Local Nonlinear Pooling}\label{sec:local_max_pooling}
Often, nonlinear pooling operations are applied to downsample feature maps. For concreteness, we will focus on one-dimensional signals in this section and relegate the proofs for two-dimensional signals (images) to Appendix \ref{app:pooling}. Let $\bm r(\bm x) \in \mathbb{R}^{N \times D}$ represent a feature map with $N$ filters and length-$D$ signals. Consider a pooling operation $\mathcal P(\cdot)$ which maps the $D$ pixels in each feature map into new vectors of size $D/k$ for some integer $k$. Note that the pooled code is equivariant to the subgroup $H = \mathbb{Z}^{D/k}$, in the sense that $\mathcal P( \pi(h) \bm r) = \rho(h) \mathcal P(\bm r)$ for any $h \in H$. The representation $\rho$ is the regular representation of the subgroup $H$. We thus decompose $G$ into cosets of size $D/k$: $g = j h$, where $j \in \cyc{k}$ and $h \in \cyc{D/k}$. The condition that a vector $\vw$ separates the dataset is
\begin{align}\label{eq:local_pooling}
    \forall \mu \in [P], j \in \cyc{k}, h\in H : y^\mu \vw^\top \rho(h) \mathcal P\left( \pi(j)  \bm r^\mu \right) > 0.
\end{align}
We see that there are effectively $k$ points belonging to each of the $P$ orbits in the pooled code.
Since $\mathcal P\left( \cdot \right)$ is nonlinear, the averaging trick utilized in Lemma \ref{thm:lemma-1} is no longer available.
However, we can obtain a lower bound on the capacity, $f(P, \floor{N/k})$, from a simple extension of Cover's original proof technique as we show in Appendix \ref{app:lower_and_upper_bound_pool}.
Alternatively, an upper bound $f \leq f(P,N)$ persists since a $\bm w$ which satisfies \Eqref{eq:local_pooling} must separate $\left< \rho(h) \right>_{h\in H} \mathcal{P}(\bm r^\mu)$. This upper bound is tight when all $k$ points $\left< \rho(h) \right>_{h\in H} \mathcal{P}(\pi(j) \bm r^\mu)$ coincide for each $\mu$, giving capacity $f(P,N)$.
This is what occurs in the average pooling case where the upper bound $f \leq f(P,N)$ is tight. Further, if we are only interested in the $H$-invariant capacity problem, the fraction of separable dichotomies is $f(P,N)$, since $\rho$ is a regular representation of $H$ as we show in \ref{app:two_dimensional_pool}.

\begin{figure}
    \centering
    \includegraphics[width=\linewidth]{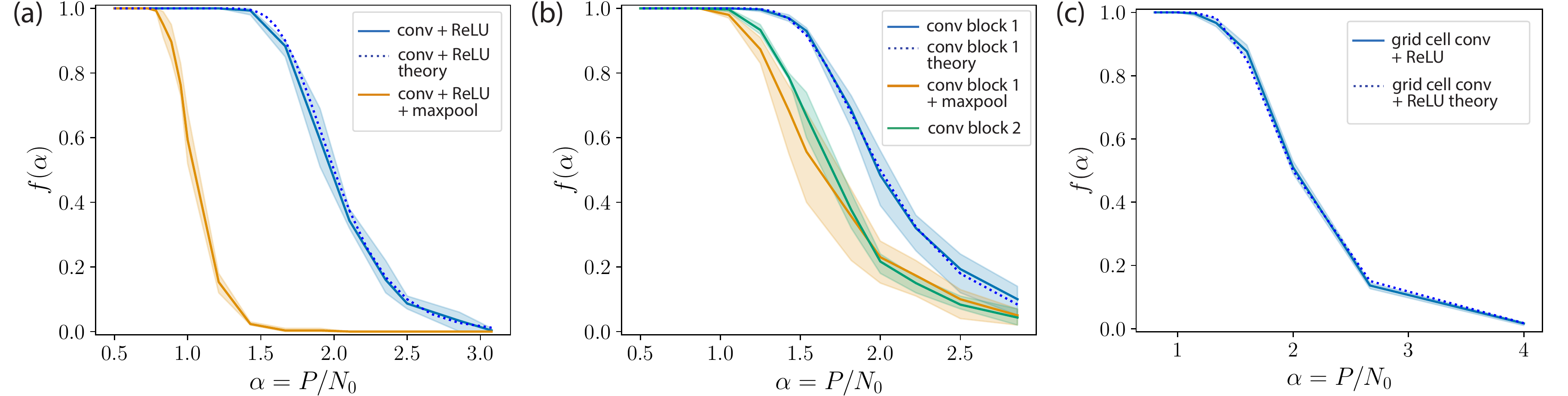}
    \caption{Capacity of GCNN representations. Solid lines denote the empirically measured fraction $f(\alpha)$ of 100 random dichotomies for which a logistic regression classifier finds a separating hyperplane, where $\alpha = P/N_0$. Dotted lines denote theoretical predictions.
    Shaded regions depict 95\% confidence intervals over random choice of inputs, as well as network weights in (a) and (c).
    (a) $f(\alpha)$ of a random periodic convolutional layer after ReLU (blue line) and followed by 2x2 max pool (orange line), with $P=40$ and $N_0 =$ \# output channels.
    Max pooling reduces capacity by a factor between 1/4 and 1 as predicted by our theory.
    (b) $f(\alpha)$ of VGG-11 pretrained on CIFAR-10 after a periodic convolution, batchnorm, and ReLU (blue line), followed by a 2x2 maxpool (orange line), and then another set of convolution, batchnorm, and ReLU (green line), with $P=20$ and $N_0 =$ \# output channels.
    Max pooling reduces capacity as predicted.
    (c) $f(\alpha)$ after a random convolutional layer equivariant to the direct sum representation of $\cyc{10} \oplus \cyc{8}$ as defined in \Secref{sec:ex-direct-sum}, with $P=16$ and $N_0 =2$ (\# output channels).
    }
    \label{fig:capacity_plots}
\end{figure}

\subsubsection{Global pooling}\label{sec:results-3}
Lemma~\ref{thm:lemma-1} shows that linear separability of $\pi$-manifolds is equivalent to linear separability after global average pooling (i.e. averaging the representation over the group).
This is relevant to CNNs, which typically perform global average pooling following the final convolutional layer, which is typically justified by a desire to achieve translation invariance.
This strategy is efficient: Lemma~\ref{thm:lemma-1} implies that
it allows the learning of optimal linear readout weights following the pooling operation, given only a single point from each $\pi$-manifold (i.e. no need for data augmentation).
Global average pooling reduces the problem of linearly classifying $|G|P$ points in a space of dimension $N$ to that of linearly classifying $P$ points in a space of dimension $\dim V_0$.
Note that by an argument similar to Section \ref{sec:local_max_pooling}, global max pooling may in contrast reduce capacity.

\subsection{Induced representations}\label{sec:induced-reps}
\vspace{-.1in}
Induced representations are a fundamental ingredient in the construction of general equivariant neural network architectures~\citep{cohenGeneralTheoryEquivariant2019}.
% Given a linear transformation $L : V \to V'$ (such as the linear part of a neural network layer), this transformation is equivariant with respect to representations $\pi _1$ and $\pi _2$ if $L(\pi^{(1)}(g)x) = \pi^{(2)}(g)L(x)$ for all $g \in G$ and $x \in V$.
Here we state our result, and relegate a formal definition of induced representations and the proof of the result to Appendix \ref{app:induced-reps}.
% Here we define induced representations and state their capacity, while relegating a more intuitive discussion of these representations and a derivation of the result to Appendix \ref{app:induced-reps}.

\begin{proposition}\label{thm:induced}
Let $\pi$ be a representation of a finite group induced from $\rho$. Then the fraction of separable dichotomies of $\pi$-manifolds is equal to that  of the $\rho$-manifolds.
\end{proposition}
% \begin{proof}
% See Appendix \ref{app:induced-reps}.
% \end{proof}
% The existence of extensions beyond finite groups is not clear to the authors, but we welcome information if such exists.
% The authors are unaware of an extension of this result beyond finite $G$ but welcome information if such an extension exists.

\section{Discussion and Conclusion}
\vspace{-.1in}
Equivariance has emerged as a powerful framework to build and understand representations that reflect the structure of the world in useful ways.
In this work we take the natural step of quantifying the expressivity of these representations through the well-established formalism of perceptron capacity.
We find that the number of ``degrees of freedom'' available for solving the classification task is the dimension of the space that is fixed by the group action.
This has the immediate implication that capacity scales with the number of output channels in standard CNN layers, a fact we illustrate in simulations.
However, our results are also very general, extending to virtually any equivariant representation of practical interest -- in particular, they are immediately applicable to GCNNs built around more general equivariance relations, as we illustrate with an example of a GCNN equivariant to direct sum representations.
We also calculate the capacity of induced representations, a standard tool in building GCNNs, and show how local and global pooling operations influence capacity.
The measures of expressivity explored here could prove valuable for ensuring that models with general equivariance relations have high enough capacity to support the tasks being learned, either by using more filters or by using representations with intrinsically higher capacity.
%Pereceptron capacity and VC dimension are closely related to the ability of a model to "memorize" a dataset, and therefore relevant for determining if a network is over/under parameterized.
% that is naive to structure between inputs and class labels.
%The measures of expressivity explored here could prove valuable for ensuring that models with general equivariance relations can support the tasks being learned.
% could prove a valuable tool in building new G-CNN architectures with more sophisticated equivariance relations than those now considered.
We leave this to future work.

While the concept of perceptron capacity has played an essential role in the development of machine learning systems~\citep{scholkopfLearningKernelsSupport2002,cohenSeparabilityGeometryObject2020} and the understanding of biological circuit computations~\citep{brunel2004optimal,Chapeton2012EfficientAM,rigotti2013importance,Brunel2016IsCC,Rubin2017BalancedEA,Pehlevan2017ResourceefficientPH,lanoreCerebellarGranuleCell2021,froudarakis2020object}, more work is wanting in linking it to other computational attributes of interest such as generalization.
A comprehensive picture of the computational attributes of equivariant or otherwise structured representations of artificial and biological learning systems will likely combine multiple measures, including perceptron capacity.
There is much opportunity, and indeed already significant recent work (e.g.~\cite{sokolic17a}).
% Generalized convolutional neural networks built around different group equivariance relations are 
% \clearpage

\section*{Reproducibility Statement}
% Code for the simulations is attached as supplementary material zip file.
Code for the simulations can be found at \url{https://github.com/msf235/group-invariant-perceptron-capacity}.
This code includes an environment.yml file that can be used to create a python environment identical to the one used by the authors.
The code generates all of the plots in the paper.

While the main text is self-contained with the essential proofs, proofs of additional results and further discussion can be found in the Appendices below.
These include
\begin{description}
    \item[Appendix \ref{app:glossary}] a glossary of definitions and notation.
    \item[Appendix \ref{app:cyclic-irreps}] a derivation of the irreps for the regular representation of the cyclic group $Z_{m}$.
    \item[Appendix \ref{app:so3-example}] a description of the irreps for the group $SO(3)$ and the resulting capacity.
    \item[Appendix \ref{app:gcnn}] a description of the construction of the GCNNs used in the paper, the methods for empirically measuring the fraction of linearly separable dichotomies, a description of local pooling, and complete formal proofs of the fraction of linearly separable dichotomies for these network representations (including with local pooling). This appendix also includes more description of the induced representation and a formal proof of the fraction of linearly separable dichotomies.
\end{description}
% These include a description of the models and process for finding the separating hyperplanes.

Appendix \ref{app:standard-cnn} also contains an additional figure, \Figref{fig:non-periodic-conv}.

\section*{Acknowledgements}
MF and CP are supported by the Harvard Data Science Initiative. MF thanks the Swartz Foundation for support. BB acknowledges the support of the NSF-Simons Center for Mathematical and Statistical Analysis of Biology at Harvard (award \#1764269) and the Harvard Q-Bio Initiative. ST was partially supported by the NSF under grant No. DMS-1439786.
% \mc{note to add Swartz funding}

\clearpage

\bibliography{manuscript.bib}
% \nobibliography{manuscript.bib}
\bibliographystyle{ICLR/iclr2022_conference}

\clearpage
\appendix
\counterwithin{figure}{section}

\section{Appendix}

\subsection{Notation and Glossary}\label{app:glossary}

\begin{itemize}
    \item $\vx$: an abstract notation for an input object
    \item $\vr(\vx)$: a feature map of the input to an $N$ dimensional vector space.
    \item $\pi$: $N$ dimensional linear representation of group $G$. For each $g \in G$, $\pi(g) \in GL(\mathbb{R}^N)$ is an $N\times N$ invertible real matrix.
    \item Equivariance property: $\vr(g \vx) = \pi(g) \vr(\vx)$ for all $g \in G$ and all $\vx$.
    \item Invariant measure: a measure $\mu: G \to \mathbb{R}_+$ on $G$ with $\mu(g S) = \mu(S) = \mu(Sg)$. For finite groups, the uniform distribution. For locally compact topological groups, the Haar measure.
    \item $\left< \cdot \right>_{g \in G}$: an average over the invariant measure of $G$
    \item Irreducible representation (irrep): an irreducible representation $\rho$ on vector space $V$ satisfies $\rho(g) v \in  V$ for all $v \in  V, g \in G$.  
    \item Character $\chi(g)$: the trace of the representation $\chi(g) = \text{Tr} \ \pi(g)$.
    \item Fixed point subspace: the subspace $V_0$ for which $\pi(g) v \in V_0$ for all $v \in V_0$. 
    \item General position: a collection of $P$ points in general position in an $N$ dimensional vector space have the property that any subset of $k \leq N$ points are linearly independent. These points are generic in the sense that they satisfy no more linear relationships than they must.
    \item Dichotomy: a particular binary labeling $\{ y^\mu \}$ of $P$ points $\{\vx^\mu\}$. 
    \item $f(P,N)$: fraction of linearly separable dichotomies given by Cover's function counting theorem \citep{Cover1965GeometricalAS}.
    \item VC dimension: the largest possible integer $P$ such that there exist $P$ points where all possible dichotomies $\{y^\mu\}$ can be realized by the model \cite{abu2012learning}.
    \item Capacity: the largest possible ratio $\alpha_c = P/N$ where $P$ points in general position can be linearly separated by a $N$ dimensional perceptron with probability $1$ in an asymptotic limit where $P,N \to \infty$ with $P/N = O_{N,P}(1)$. The classical result is $\alpha_c = 2$ \citep{Gardner1987MaximumSC, shcherbina_rigorous_2003}.
    \item $\mathcal{P}(\cdot)$: a local pooling operation.
\end{itemize}

\subsection{Irreps for the cyclic group}\label{app:cyclic-irreps}
Here we compute the irreps of representations $\pi: \cyc{m} \to \gl (V)$ of the cyclic group $\cyc{m}$ over a real vector space $V$ (see~\citep{serreLinearRepresentationsFinite2014} for a derivation of the irreps when $\mV$ is a complex vector space).
To find the irreps, one can use the form for the eigenvalues and eigenvectors for circulant matrices, since all the $\pi (g)$ are circulant.
This results in the simultaneous diagonalization
$\pi (g) = \mV \left(1 \oplus \mR(2 \pi g / m) \oplus \mR(4 \pi g / m) \oplus \cdots \oplus \mR((m-1) \pi g / m)\right) \mV^{\top}$
where $\mV$ is the real-valued version of the discrete Fourier transform matrix (the columns are proportional to cosines and sines of varying frequencies, along with a column proportional to $\vone_m$).
% \mc{cite}.

% one can use the real discrete Fourier transform matrix.
Note that the $2x2$ rotation matrices $\mR (2 \pi g k / m)$
are irreps for $k \neq m/2$ and $k\neq 0$, since there is no one-dimensional subspace of $\sR^{2}$ that is invariant to $\mR (2 \pi g k / m)$ for all $g$.
The exception for $k = m/2$ if $m$ is even, gives
$\mR (2 \pi g k / m) = (-1)^g \mI$, which corresponds to rotation of 180 degrees.
The subspace $\vspan\{(1,0)\}$ is invariant to this action, so that  $\mR (2 \pi g k / m)$ is not an irrep.
This representation can thus be reduced to an action on a one-dimensional subspace, represented by $(-1)^g$.
% However, this representation can be decomposed into a direct sum of two copies of the one-dimensional representation $(-1)^g$, corresponding to the way 180 degree rotations flip the sign of the first and second entries of length-two vectors.
The case $k=0$ gives the trivial representation.
% \mc{need to check how many copies of this representation are there really}

% Finally, note that $\vspan \{ \vone_m \}$ is the only stabilizer subspace of $\pi$.

% \begin{align*}\mR (2 \pi g k / m) =\begin{bmatrix}(-1)^g & 0 \\ 0 & 1 \end{bmatrix}\end{align*}
% The nontrivial 2-dimensional irreps have the form
% $\rho_j (k) = \mR(j \theta_k)
% %\begin{bmatrix} \cos(j\theta_k) &  -\sin(j\theta_k) \\
% %\sin(j\theta_k) & \cos(j\theta_k)\end{bmatrix}
% $
% where $\theta_k = 2 \pi k / m$ for $k\in \{0,1,...,m-1\}$ and $j \in \{1,...,m-1\}$ (see Appendix \ref{app:cyclic-irreps} for a derivation).
% When $m$ is even then there is a single one-dimensional nontrivial irrep, which captures the effect of rotating by 180 degrees: $\rho_{m/2}(k) = (-1)^k$.\cep{Is this obvious, or needs to be proved/cited?}.
% \mc{the references I'm aware of consider irreps over the complex numbers so we may need to explain these irreps in the appendix.}\cep{Please do!}
% The final irrep is the trivial, 1-dimensional representation $\rho_0 (k) = 1$. Hence the capacity is $f(P,1)$.

\subsection{SO(3): A Non-Abelian Lie Group}\label{app:so3-example}

The special orthogonal group $\text{SO}(3)$ on $3$ dimensions (rotation group), the $3\times 3$ orthogonal matrices with determinant $+1$, can also be analyzed within our theory. G-convolutional neural networks that are equivariant to $\text{SO}(3)$ rotations have become of high interest in the physical sciences and computer vision where the objects of interest often respect these symmetries~\citep{andersonCormorantCovariantMolecular2019, Cohen2018SphericalC, Esteves2018LearningSE, Kondor2018ClebschGordanNA}
The irreducible representations have the form $\mB_{k_m}$
% \begin{align}
%     \pi(g) = \begin{bmatrix}
%                 \bm B_{k_1}(g) & 0 & .... & 0
%                 \\
%                 0 & \bm B_{k_2}(g) & ... & 0
%                 \\
%                 \vdots & ... & \ddots & 0  
%                 \\
%                 0 & 0 & ... & \bm B_{k_M}(g)
%                 \end{bmatrix}
% \end{align}
where $\bm B_{k_m}$ are $( 2k_m +1) \times (2k_m + 1)$ block matrices, known as Wigner $D$-matrices~\citep{wignerGruppentheorieUndIhre1931}. The trivial irreps correspond to the one-dimensional irreps with $k=0$. Thus, the $\text{SO}(3)$-invariant classification capacity merely counts the number of trivial irreps which have $N_0 = \sum_{m} \delta_{k_m,0}$. The capacity is again $f(P,N_0)$.

\subsection{G-equivariant convolutional layers}\label{app:gcnn}
\subsubsection{Standard convolutional layers}\label{app:standard-cnn}
A convolutional layer consists of a set of $N$ $k\times k'$ filters $F_{i}$ that are convolved (technically cross-correlated) with a stack of $M$ $W \times L$ input tensors.
Here $M$ is the number of input channels and $N$ the number of output channels.
The convolution runs each filter (i.e. takes the dot product at all possible positions) over each of the $W \times L$ input tensors, and the result is averaged across the $M$ input channels to produce the output of one output channel.
In the positions where the filters approach the edges of the input tensor, different choices can be made about how to handle these edge conditions.
The standard choice is to pad the edges with some number of zeros depending on the desired shape of the output tensor and run the convolution out to the end of the padded image.
Another possible choice is to loop the edges of the input tensor together, so that the filter is applied to the other side of the input tensor as it runs off the edge.
This periodic boundary condition allows us to write the convolution formally in terms of group actions, and to apply our theory directly.
When convolutions are not periodic, the resulting capacity increases somewhat but still follows the $P/N_0$ scaling of the periodic convolutions (\Figref{fig:non-periodic-conv}).

For the random convolutional layers of \Figref{fig:capacity_plots}a, the input tensors are size $10 \times 10$ and the number of input channels are $M=3$, as for standard color images.
Each entry of these tensors is normally distributed with mean $0$.
The filters are also of size $10 \times 10$ with periodic boundary conditions, and are initialized according to a normal Xavier distribution with parameters that are the default for Pytorch 1.9.
The convolution is implemented via the Pytorch 1.9 implementation of Conv2d with padding\_mode=``circular'' and padding=0 in the case of periodic boundary conditions.
The bias term of the convolution is set to zero and the convolution is followed by a ReLU nonlinearity (blue line).
The resulting figures do not change appreciably for different choices of input tensor size, number of input channels, or size of filters (though note that the nonlinearity is essential for satisfying the general position condition of Theorem \ref{thm:thm-1}; otherwise, the capacity would be determined by the number of input channels rather than number of output channels).
The output of this convolution is then fed through a $2\times2$ max pooling layer (orange line in \Figref{fig:capacity_plots}a), provided by Pytorch 1.9's MaxPool2d.

The pretrained VGG-11 layers used in \Figref{fig:capacity_plots}b and \Figref{fig:non-periodic-conv} are taken from~\citet{kuangliu2021}.
The first convolutional block (blue line) consists of $3 \times 3$ pretrained filters applied to CIFAR-10 image tensors randomly selected from the validation set and normalized in the same way they are normalized during training (see~\citet{kuangliu2021} for details).
These images are of size $32 \times 32$ and have $M=3$ input channels, followed by a batch normalization layer in evaluation mode (fixed parameters), and then followed by a ReLU nonlinearity.
The boundary conditions of these convolutions are set to be periodic in \Figref{fig:capacity_plots}b and nonperiodic with a zero padding sizes of 1 in \Figref{fig:non-periodic-conv}, and the bias term is set to zero.
The batch normalization is an element-wise operation and so equivariant to the representations we consider -- thus this operation is not expected and is not observed to affect the perceptron capacity.
This convolutional block is then followed by a $2 \times 2$ max pooling layer (orange line).
Finally, another convolutional block of $3\times3$ filters, batch normalization, and ReLU nonlinearity are applied (green line).

The fraction of linearly separable dichotomies is measured empirically by using the scikit-learn LogisticRegression implementation of logistic regression, with a tolerance value of tol=1e-18 and an inverse regularization value of C=1e8.
The maximum number of iterations is set to 500.
An intercept (i.e. bias) is not used for this fit.

To formally prove that the fraction of separable dichotomies is $f(P,N)$ for standard periodic convolutional layers, first note that the convolution is equivariant with respect to cyclic permutation of the inputs and of the outputs.
The representation for cyclically permutating the output tensor can be written $\bigoplus_{k=1}^{N}\pi$ where $\pi$ is the representation that cyclically permutes the entries of $W \times L$ matrices.
Since each copy of $\pi$ contains one trivial irrep in its decomposition into a direct sum of irreps, the direct sum $\bigoplus_{k=1}^{N}\pi$ contains $N$ trivial irreps in its decomposition.
The final step to use Theorem \ref{thm:thm-1} is to argue that the centroids of the manifolds are in general position.
Since a nonlinearity (ReLU) is applied to the output of the convolution, and since there is no particular structure in the convolutional filters beyond possibly sparsity, we can generically expect this to be the case.

\begin{figure}
    \centering
    \includegraphics[width=.6\linewidth]{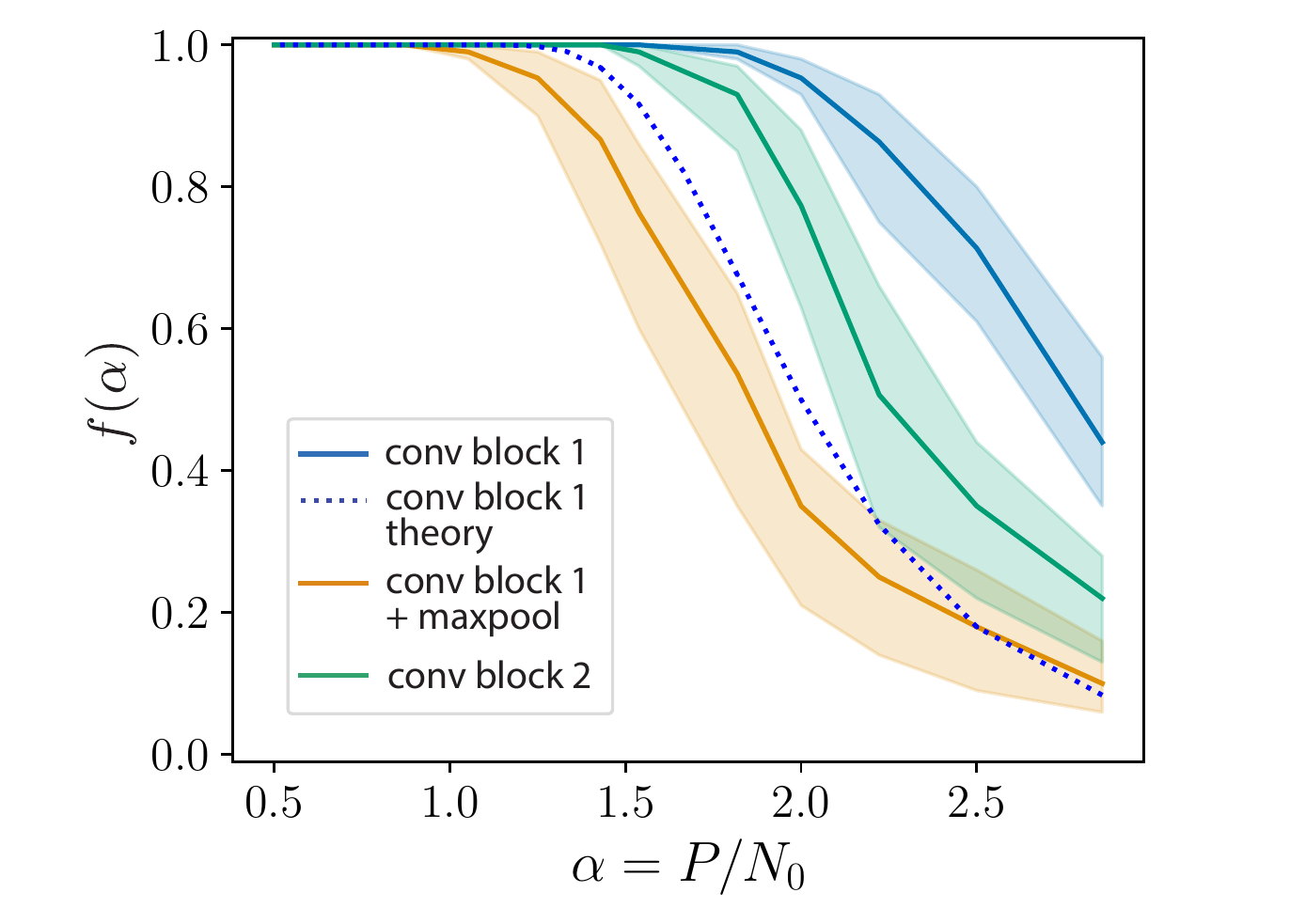}
    \caption{The fraction of realizable dichotomies of non-periodic convolutional layers is higher than periodic convolutional layers, but still obeys the same scaling. Details are exactly as in \Figref{fig:capacity_plots}b, but using non-periodic convolutions with a zero padding of size 1.
    Here the theory line refers to the theory for periodic convolutions.}
    \label{fig:non-periodic-conv}
\end{figure}

\subsubsection{Local Pooling}\label{app:pooling}
First we prove an extension of Lemma \ref{thm:lemma-1} to equivariant linear maps. This will be used to show that average pooling does not affect the capacity of the regular representation of $\cyc{m}$.

\begin{lemma}\label{thm:lemma-A1}
 Let $\pi$ be a representation of the group $G$ and suppose the matrix $\mP$ is equivariant with respect to the restriction of $\pi$ to a subgroup $H \subseteq G$, so that for all $h\in H$ $\mP \pi(h) = \rho(h) \mP$ for some representation $\rho$ of $H$. Let $R$ denote a set of representatives of $G/H$. Then we have the following equivalence.
\begin{align*}
    \exists \vw \  &\forall \mu \in [P] , g\in G : y^\mu \vw^\top \mP \pi(g) \bm r^\mu  > 0 
    \\
    &\iff \exists \vw \  \forall \mu \in [P] \ \forall g' \in R \ : \  y^\mu \vw^\top \mP \avg{ \pi (h) }_{h\in H} \pi(g') \vr^\mu > 0.
\end{align*}
\end{lemma}
\begin{proof}
For the forward implication, we write the coset decomposition $g=hg'$ of $g$ and average over $H$ to find
\begin{align*}
    \forall g\in G : y^\mu \vw^\top \mP \pi(g) \bm r^\mu  > 0 & \iff \forall h\in H, g'\in R : y^\mu \vw^\top \mP \pi(h)\pi(g') \bm r^\mu  > 0 \\
    & \implies \forall g'\in R : y^\mu \vw^\top \mP \avg{\pi(h)}_{h\in H}\pi(g') \bm r^\mu  > 0.
\end{align*}

For the backward implication, suppose $y^{\mu}\vw^{\top} \mP \langle \pi(h) \rangle_{h \in H} \pi(g') \vr^{\mu}>0$ for all representatives $g' \in R$, and define $\tilde{\vw} = \avg{\rho(h)}_{h \in H}^{\top} \vw$. For any $g \in G$, take a coset decomposition $g = h g'$ for $h \in H$ and $g'\in R$.
We then have
\begin{align}
    y^\mu \tilde{\vw}^\top \mP \pi(g) \bm r^\mu &= y^\mu \tilde{\vw}^\top \mP \pi(h) \pi(g') \bm r^\mu \ \text{ (Coset decomposition)} \nonumber
    \\
    &= y^\mu \tilde{\vw}^\top \rho(h) \mP \pi(g') \bm r^\mu \ \text{($\mP$ is $H$-equivariant) } \nonumber
    \\
    &= y^\mu \vw^\top \left< \rho(h') \right>_{h' \in H} \rho(h) \mP \pi(g') \bm r^\mu \ \text{(Definition of $\tilde\vw$)} \nonumber
    \\
    &= y^\mu \vw^\top \left< \rho(h) \right>_{h \in H} \mP \pi(g') \bm r^\mu  \ \text{(Invariance of measure)} \nonumber
    \\
    &= y^\mu \vw^\top  \mP \left< \pi(h) \right>_{h \in H} \pi(g') \bm r^\mu \nonumber \ \text{($\mP$ is linear and $H$-equivariant) }
    \\
    &> 0 \  \text{(By assumption)}.
\end{align}
The implication follows.
%We will show that $\tilde{\vw}$ separates the $P$ manifolds $\{\mP \pi(g) \vr^{\mu} : g \in G\}_{\mu \in [P]}$ since 
%\begin{align*}
%    y^{\mu} \tilde{\vw}^{\top} \mP \pi (g)\vr^{\mu}&
%    =y^{\mu} \vw^{\top} \avg{\rho} \mP \pi (g)\vr^{\mu} \ \  \text{(Definition of $\tilde{\vw}$)} 
%\\
%&= y^{\mu} \vw^{\top}\avg{\rho(g') \mP }_{g'\in G}\pi (g)\vr^{\mu}  \ \  \text{(Definition of $\avg{\rho}$ and linearity of $\mP$)}
%\\
%&= y^{\mu} \vw^{\top}\avg{ \mP \pi(g') }_{g'\in G}\pi (g)\vr^{\mu}  \ \  \text{(Equivariance of $\mP$)}
%\\
%&= y^{\mu} \vw^{\top}\mP\avg{ \pi(g'g)}_{g'\in G}\vr^{\mu} \ \  \text{($\pi(g)$ is a representation of $G$)} 
%\\
%&= y^{\mu} \vw^{\top}\mP \avg{\pi}\vr^{\mu} \ \  \text{(Invariance of the Haar Measure $\mu(S g) = \mu(S)$ for set $S$)} 
%\\
%&>0 \ \ \ \  \text{(Initial assumption)}.
%\end{align*}
\end{proof}

\begin{lemma}
 For the regular representation of $G = \cyc{D}$, a local average pooling over windows of size $k$ generates a matrix $\mP$ which is equivariant with respect to the subgroup $H = \cyc{D/k}$ with the property that
 \begin{align}
     \mP \avg{ \pi(h) }_{h \in H} = a \mP \avg{ \pi(g) }_{g \in G}
 \end{align}
 where $a>0$ is a positive constant.
\end{lemma}
\begin{proof}
    First, note that the new pooled code is the regular representation of $H$ since shifts of size $mk$ in the original feature map corresponds to shifts of length $m$ in the pooled code. Thus $\mP$ is equivariant to $H = \cyc{D/k}$.
    Next we note
    % that $\avg{\pi(h)}_{h\in H}$ is proportional to the matrix
    % \begin{align*}
    %     \begin{bmatrix}
    %         1 & 0 \cdots & 0 & 1 & 0 & \cdots & 0 & \cdots \\
    %         0 & 0 \cdots 
    %     \end{bmatrix}
    % \end{align*}
    the following two facts 
    \begin{align}
        \mP \left< \pi(h) \right>_{h\in H} = a' \bm 1_{D/k} \bm 1_{D}^\top
        \\
        \mP \left< \pi(h) \right>_{g\in G} = a'' \bm 1_{D/k} \bm 1_D^\top
    \end{align}
    $a'$ and $a''$ are positive constants and $\vone_D$ and $\vone_{D/k}$ are the $D$ and $D/k$ dimensional vector of all ones, respectively.
    Thus, we have that $\mP \avg{ \pi(h) }_{h\in H} = a\mP \avg{ \pi(g) }_{g \in G}$ where $a$ is a positive constant.
    % \mc{I think to use this below we need to say more, that the constant of proportionality is positive.}
\end{proof}

\begin{theorem}\label{app:thm_avg_pool_capacity}
    The fraction of linearly separable dichotomies of a CNN pooling layer with $N$ filters after average pooling from feature maps with the size of the input image $W \times L$ to pooled feature maps of size $W / k \times L / k$ is $f(P,N)$, i.e. no capacity is lost due to local average pooling.
\end{theorem}
\begin{proof}
    The CNN layer before pooling is a regular representation $\pi$ of the full group $G = \cyc{W} \times \cyc{L}$,
    applied to each of the $M$ input channels via a direct a sum $\bigoplus_{j=1}^{M} \pi$. The layer after pooling is a regular representation $\rho$ of the subgroup $H = \cyc{W/k} \times \cyc{L/k}$,
    also applied to the output channels via a direct sum $\bigoplus_{j=1}^{N} \rho$.
    Let $R$ be a set of representatives of $G/H$.
    Since $\mP$ is equivariant to $\pi$ and $\rho$ over $H$
    % and $\mP \avg{\pi(h)}_{h \in H} = a\mP \avg{\pi(g)}_{g \in G}$ where $a>0$,
    % and using that $\vone_D^{\top}\pi(g) = \vone_D^{\top}$ for all $g$,
    we have by the previous two lemmas that
    \begin{align*}
        \forall g\in G : y^\mu \vw^\top \mP  \pi(g) \bm r^\mu > 0 
        & \iff \forall  g'\in R : y^{\mu } \vw^\top \mP \avg{\pi(h)}_{h \in H} \pi(g') \bm r^\mu > 0 \\
        % & \iff \forall_{g'\in R} y^{\mu } \vw^\top \vone_{D/k} \vone^{\top}_{D} \pi(g') \bm r^\mu > 0 \\
        % & \iff \forall_{g'\in R} y^{\mu } \vw^\top \vone_{D/k} \vone^{\top}_{D} \pi(g') \bm r^\mu > 0 \\
        % & \iff \forall_{g'\in R} y^{\mu } \vw^\top \vone_{D/k} \vone^{\top}_{D} \bm r^\mu > 0
        & \iff \forall g'\in R : y^{\mu } \vw^\top \mP \avg{ \pi(g)}_{g \in G} \pi(g') \bm r^\mu > 0 \\
        & \iff y^{\mu } \vw^\top \mP \avg{ \pi(g)}_{g \in G} \bm r^\mu > 0 \\
        & \iff y^{\mu } \vw^\top \mP \avg{ \pi(h)}_{h \in H} \bm r^\mu > 0 \\
        & \iff y^\mu \vw^\top \avg{ \rho(h)}_{h \in H} \mP \bm r^\mu > 0
    \end{align*}
    Thus the capacity is determined by the rank of
    % $\mP\avg{\pi}_{g \in G}$
    $\avg{\rho(h)}_{h\in H}$, assuming the $\avg{\rho(h)}_{h\in H} \mP \vr^{\mu}$ are in general position.
    Since each $\rho$ is a copy of the regular representation for $H$, the rank of $\avg{ \bigoplus_{j=1}^{N} \rho(h) }_{h\in H}$ is merely $N$.
    Thus the fraction of linearly separable dichotomies is $f(P,N)$, the same as the capacity before pooling.
\end{proof}

Now we prove local pooling operations in a standard CNN preserve a regular representation of a subgroup of the cyclic group.
\begin{lemma}\label{app:two_dimensional_pool}
 Suppose $P$ is a local pooling operation on two-dimensional signals (CNN feature maps), and that $\pi$ is the regular representation of a group $G = \cyc{W} \times \cyc{L}$ on code $\bm a(\vx)$. A pooled feature map $\bm r = \mathcal P(\bm a)$ which acts on $k \times k$ windows of $\bm a$ is a regular representation of the subgroup $H=\cyc{W / k } \times \cyc{L / k}$.
\end{lemma}
\begin{proof}
Suppose an equivariant feature map $\bm a(\bm x) \in\mathbb{R}^{W \times L \times N}$ has corresponding regular representation of the group $G = \cyc{W} \times \cyc{L}$ for each of the $N$ filters. Consider any local pooling operation $\mathcal P\left( \cdot \right)$ (such as average or maximum) which acts on $k \times k$ patches where $k$ divides both $W$ and $L$.
\begin{align}
    r_{ij,h}(x) = \mathcal P\left( \{ a_{i',j',h}(x) \ | \ i' \in [i k,(i+1) k], j' \in [j k,(j+1)k]  \}  \right)
\end{align}

Note that for $k>1$, $\bm r(x)$ is no longer equivariant to $G$ since the representation does not satisfy the homomorphism property for shifts with length $\ell$ not divisible by $k$. However, the new code is equivariant to a subgroup $H = \cyc{W/k} \times \cyc{L/k}$, namely vertical and horizontal shifts with length divisible by $k$. Let $\bm\pi^{\vx}_{nk,mk}$ represent a vertical shift of the image $\bm x$ by $nk$ pixels and horizontal shift by $mk$ pixels. Note that $a_{ij,h}(\bm\pi^{\vx}_{nk,mk} \vx) = a_{i+nk,j+mk}(\vx)$ since $\bm a(\vx)$ is equivariant.  Then, the $h$-th pooled feature map transforms as 
\begin{align*}
    r_{ij,h}(\bm\pi^{\vx}_{nk,mk} \bm x) &= \mathcal P\left( \{   a_{i',j',h}( \bm\pi^{\vx}_{nk,mk} x) \ | \ i' \in [ik,(i+1)k], j' \in [jk,(j+1)k]  \} \right) 
    \\
    &= \mathcal P\left( \{   a_{i'+nk,j'+mk,h}(\bm x) \ | \ i' \in [ik,(i+1)k], j' \in [jk,(j+1)k]  \} \right) \\
    & \qquad \quad \text{($\bm a$ is Equivariant to $\pi ^{x}$)} 
    \\
    &= \mathcal P\left( \{   a_{i',j',h}(\bm x) \ | \ i' \in [(n+i)k,(n+i+1)k], j' \in [(m+j)k,(m+j+1)k]  \} \right),
    \\ & \qquad \quad \text{($k$ divides $W,H$)} 
    \\
    &= r_{i+n,j+m,h}(\bm x) \ , \ \text{(Definition of $\bm r$)}
\end{align*}
We thus find that the code is a regular representation of the subgroup of the cyclic translations $H = \{(nk,mk)\}_{n \in [H/k], m\in [W/k]}$. This new group $G'$ has dimension $|H| = \frac{1}{k^2} |G|$. 
\end{proof}

\subsubsection{Lower Bound and Upper Bound on Capacity for Nonlinear Pooling}\label{app:lower_and_upper_bound_pool}

\begin{theorem}\label{app:lower_bound_pooling}
    Suppose a code which is equivariant to a finite group $G$ is pooled to a new code which is equivariant to a finite subgroup $H \subseteq G$. Suppose the number of trivial dimensions in the original $G$-equivariant code is $N_0$. Then the fraction of linearly separable dichotomies on the $G$-invariant problem for the pooled code is at least $f(P, \floor{N_0 / k})$ where $k = |G / H|$. Similarly the fraction is at most $f(P,N_0)$.
\end{theorem}

\begin{proof}
The pooled code, by assumption, has the property $\mathcal P(\pi(h g') \vr) = \rho(h) \mathcal P\left( \pi(g')\vr \right)$ for any $h \in H$ and $g' \in R$, where $R$ is a set of representatives of $G/H$.
The $G$-invariant separability condition amounts to the proposition
\begin{align}
    \exists \vw &\forall \mu \in [P] \ \forall h \in H, g' \in R : y^\mu \vw^\top \rho(h) \mathcal P( \pi(g') \vr^\mu ) > 0 
    \\
    &\iff \exists \vw \forall \mu \in [P] \ \forall g' \in R : y^\mu \vw^\top \left< \rho(h) \right>_{h \in H} \mathcal P( \pi(g') \vr^\mu ) > 0;
\end{align}
in other words, a solution on the right hand side affords a solution over all of the manifolds generated in the input space.
We see that, this requires considering if this particular dichotomy is linearly separble on the $P k$ anchor points $\left< \rho(h) \right>_{h\in H} \mathcal P(\pi(g') \vr^\mu)$. The simplest strategy to obtain an upper bound is to consider what happens when a single new manifold is added. We see that when a single new base point $\vr$ is added it corresponds to $k$ new points in the orbit $\left<\rho\right>\mathcal P(\pi(g') \vr)$ for all $g' \in R$. Suppose that $\left< \rho(h) \right>_{h \in H}$ has rank $N_0$. Let $C(P,N_0)$ represent the number of linearly separable dichotomies for $P$ $G$-orbits in $N_0$ trivial dimensions. Upon the addition of the $k$ new points ($P \to P+1$), we find that some of the pre-existing separable dichotomies give a new separable dichotomy. This can be guaranteed to occur when a $\vw$ separates the old dichotomy and has $\vw^{\top} \left< \rho \right> \cdot \mathcal P(\pi(g') \vr) = 0$ for the new anchor point $\vr$ (but this condition is not \textit{necessary} for a new dichotomy to be separable). This condition means that the original dichotomy is separable in the $N_0 - k$ dimensional subspace $\{\vw : \vw \cdot \mathcal P(\pi(g') \vr) = 0 \}$. By making infinitesimal adjustment to this $\vw$ the correct label on this new orbit can be achieved without altering the labels on any other dichotomy. Since this argument gives a sufficient but not necessary condition to generate a new dichotomy, we obtain the following inequality
\begin{align}
    C(P+1,N_0) \geq C(P,N_0) + C(P,N_0-k) .
\end{align}
Solving this recursion gives the capacity $f_{Pooled}(P,N_0) \geq f_{Cover}(P, \floor{N_0/k})$.
The greatest capacity occurs in the special case where $\left<\rho \right> \mathcal P(\pi(g') \vr) = \left< \rho \right> \mathcal P(\vr)$. In this case, the usual counting theorem applies giving a fraction of separable dichotomies of $f(P,N_0)$. This is achieved, for instance in average pooling as we showed in Theorem \ref{app:thm_avg_pool_capacity}.
\end{proof}

\subsubsection{Direct sum equivariant
convolutional layers}\label{app:direct-sum-cnn}
Here we describe how to build a convolutional layer architecture that is equivariant with respect to the regular representation in the input space and the direct sum representations introduced in \Secref{sec:ex-direct-sum} in the output space.
% For simplicity we only built these layers to take in vector data, but the generalization to two-dimensional (such as image) data is straightforward.
For the following we assume that $m_1$ and $m_2$ are coprime, though see the footnote in \Secref{sec:ex-direct-sum} for a discussion of how to loosen this requirement.

The input data is a $W \times L \times M$ tensor where $M$ is the number of input channels.
The first step is to simply take the output of a standard convolution (in our simulations we also apply a ReLU nonlinearity) applied to this input with periodic boundary conditions, resulting in a $W \times L \times N$ tensor where $N$ is the number of output channels.
The next step is to, for each of the output channels, take an average (or maximum) between entries spaced $m_1$ entries apart horizontally or vertically in the matrix, resulting in an $m_1 \times m_1$ matrix.
In our simulations we took averages rather than maximums.
This is repeated for the other number $m_2$, resulting in an $m_2 \times m_2$ matrix.
This is repeated for every output channel, resulting in $N$ matrices of size $m_1 \times m_1$ and $N$ matrices of size $m_2 \times m_2$.
Finally, the resulting matrices are flattened and appended into an $(m_1^2+m_2^2) \times N$ matrix, and the result is passed through a nonlinearity (ReLU).

As the input tensor is cyclically permuted according to a regular representation $\pi$ of $\cyc{m_1 m_2}$, the output of this equivariant convolutional layer permutes according to the representation $\pi^{(1)} \oplus \pi^{(2)}$ where $\pi^{(1)}$ is the regular representation of $\cyc{m_1}$ and $\pi^{(2)}$ is the regular representation of $\cyc{m_2}$.

The proof that the fraction of separable dichotomies is given by $f(P, 2N)$ follows the same proof as for the standard periodic convolutions in Appendix \ref{app:standard-cnn}.
Instead of a direct sum $\bigoplus_{k=1}^{N}\pi$ we get a direct sum $\bigoplus_{k=1}^{N}(\pi^{(1)} \oplus \pi^{(2)})$.
Each of the $\pi^{(1)} \oplus \pi^{(2)}$ contain two trivial irreps in its decomposition, so that the final fraction is $f(P, 2N)$.

\subsubsection{Induced representations}\label{app:induced-reps}
First we state the definition of an induced representation.
Let $H$ be a subgroup of a finite group $G$ and let $\rho : H \to \gl (W)$ be a representation of $H$.
Let $V$ be the vector space of functions $f : G \to W$ such that $f(gh) = \rho (h) f(g)$ for all $h \in H$ and $g \in G$.
We now define the induced representation $\pi : G \to \gl (V)$ to be the representation which satisfies $(\pi (g') f) (g) = f (gg')$.

For intuition, note that every element of $g$ can be written $g = rh$ where $r$ is a representative for a coset in $G/H$ and $h \in H$.
This is because the cosets $G/H$ partition $G$ and the action of $H$ stays within a coset; hence $r$ selects out the coset, and $h$ goes to the desired element of the coset: $g = rh$.
With this decomposition, the action of $\pi$ is then $\pi (rh) f (g) = \rho (h) f (gr)$.
Hence the $r$ component under $\pi$ has the effect of permuting to the new coset that $gr$ belongs in, and the $h$ component under $\pi$ then has the effect $\rho (h)$ on the resulting vector $f(gr)$ that we originally specified.
This is the most natural way to get a representation of $G$ from a representation of $H$.
In the case of finite groups, one can think of the $r$ component as permuting a set of isomorphic copies of $V$, each copy corresponding to a different coset.

To compute the capacity of induced representations, and so prove Proposition \ref{thm:induced}, we use Frobenius reciprocity of characters.
Recall that the character $\theta$ of a representation $\pi: G \to \gl (V)$ is the map $\theta: G \to \sC$ induced by the trace: $\theta (g) = \Tr (\pi (g))$.
Now let $\theta$ be the character of $\rho : H \to \gl (V)$ and let $\theta^G$ be the character of the induced representation.
Then $\langle \theta^G(g)\rangle_{g\in G} = \langle \theta(h)\rangle_{h\in H}$ by Frobenius reciprocity of characters~\citep{Mackey1970}.
The average of the character is the number of trivial representations contained in the decomposition of the representation (see ~\cite{serreLinearRepresentationsFinite2014}).
% \mc{cite, check about group needing to be finite}.
Hence the capacity of the induced representation is equal to the capacity of $\rho$.
The existence of extensions beyond finite groups is not clear to the authors, but we welcome information if such exists.
\end{document}

%% file: ICLR/math_commands.tex
\usepackage{amsmath,amsfonts,bm}

\newcommand{\newterm}[1]{{\bf #1}}

\def\Figref#1{Figure~\ref{#1}}

\def\Secref#1{Section~\ref{#1}}
\def\eqref#1{equation~\ref{#1}}
\def\Eqref#1{Equation~\ref{#1}}
\def\Lemref#1{Lemma~\ref{#1}}
\def\Thmref#1{Theorem~\ref{#1}}

\def\floor#1{\lfloor #1 \rfloor}
\def\avg#1{\langle #1 \rangle}
\def\1{\bm{1}}

\def\vzero{{\bm{0}}}
\def\vone{{\bm{1}}}

\def\va{{\bm{a}}}

\def\vr{{\bm{r}}}

\def\vv{{\bm{v}}}
\def\vw{{\bm{w}}}
\def\vx{{\bm{x}}}

\def\mB{{\bm{B}}}

\def\mI{{\bm{I}}}

\def\mP{{\bm{P}}}

\def\mR{{\bm{R}}}

\def\mV{{\bm{V}}}

\DeclareMathAlphabet{\mathsfit}{\encodingdefault}{\sfdefault}{m}{sl}
\SetMathAlphabet{\mathsfit}{bold}{\encodingdefault}{\sfdefault}{bx}{n}

\def\sC{{\mathbb{C}}}

\def\sR{{\mathbb{R}}}

\DeclareMathOperator{\Tr}{Tr}
\DeclareMathOperator{\vspan}{span}
\DeclareMathOperator{\range}{range}

%% file: manuscript_arxiv.bbl
\begin{thebibliography}{96}
\providecommand{\natexlab}[1]{#1}
\providecommand{\url}[1]{\texttt{#1}}
\expandafter\ifx\csname urlstyle\endcsname\relax
  \providecommand{\doi}[1]{doi: #1}\else
  \providecommand{\doi}{doi: \begingroup \urlstyle{rm}\Url}\fi

\bibitem[Abu-Mostafa(1993)]{MostafaHints}
Yaser~S. Abu-Mostafa.
\newblock {Hints and the VC Dimension}.
\newblock \emph{Neural Computation}, 5\penalty0 (2):\penalty0 278--288, 03
  1993.
\newblock ISSN 0899-7667.
\newblock \doi{10.1162/neco.1993.5.2.278}.
\newblock URL \url{https://doi.org/10.1162/neco.1993.5.2.278}.

\bibitem[Abu-Mostafa et~al.(2012)Abu-Mostafa, Magdon-Ismail, and
  Lin]{abu2012learning}
Yaser~S Abu-Mostafa, Malik Magdon-Ismail, and Hsuan-Tien Lin.
\newblock \emph{Learning from data}, volume~4.
\newblock AMLBook New York, NY, USA:, 2012.

\bibitem[Anderson et~al.(2019)Anderson, Hy, and
  Kondor]{andersonCormorantCovariantMolecular2019}
Brandon Anderson, Truong~Son Hy, and Risi Kondor.
\newblock Cormorant: Covariant molecular neural networks.
\newblock In H.~Wallach, H.~Larochelle, A.~Beygelzimer, F.~{dAlch{\'e}-Buc},
  E.~Fox, and R.~Garnett (eds.), \emph{Advances in Neural Information
  Processing Systems}, volume~32. {Curran Associates, Inc.}, 2019.

\bibitem[Baek et~al.(2021)Baek, Dimaio, Anishchenko, Dauparas, Ovchinnikov,
  Lee, Wang, Cong, Kinch, Schaeffer, Mill{\'a}n, Park, Adams, Glassman,
  DeGiovanni, Pereira, Rodrigues, van Dijk, Ebrecht, Opperman, Sagmeister,
  Buhlheller, Pavkov-Keller, Rathinaswamy, Dalwadi, Yip, Burke, Garcia,
  Grishin, Adams, Read, and Baker]{Baek2021AccuratePO}
Minkyung Baek, Frank Dimaio, Ivan~V. Anishchenko, Justas Dauparas, Sergey
  Ovchinnikov, Gyu~Rie Lee, Jue Wang, Qian Cong, Lisa~N. Kinch, R.~Dustin
  Schaeffer, Claudia Mill{\'a}n, Hahnbeom Park, Carson Adams, Caleb~R.
  Glassman, Andy~M. DeGiovanni, Jose~H. Pereira, Andria~V. Rodrigues,
  Alberdina~Aike van Dijk, Ana~C Ebrecht, Diederik~Johannes Opperman, Theo
  Sagmeister, Christoph Buhlheller, Tea Pavkov-Keller, Manoj~K. Rathinaswamy,
  Udit Dalwadi, Calvin~K. Yip, John~E. Burke, K.~Christopher Garcia, Nick~V.
  Grishin, Paul~D. Adams, Randy~J. Read, and David Baker.
\newblock Accurate prediction of protein structures and interactions using a
  three-track neural network.
\newblock \emph{Science}, 373:\penalty0 871 -- 876, 2021.

\bibitem[Baum(1988)]{Baum1988OnTC}
Eric~B. Baum.
\newblock On the capabilities of multilayer perceptrons.
\newblock \emph{J. Complex.}, 4:\penalty0 193--215, 1988.

\bibitem[Bekkers et~al.(2018)Bekkers, Lafarge, Veta, Eppenhof, Pluim, and
  Duits]{Bekkers2018RotoTranslationCC}
Erik~J. Bekkers, Maxime~W. Lafarge, Mitko Veta, Koen A.~J. Eppenhof, Josien
  P.~W. Pluim, and Remco Duits.
\newblock Roto-translation covariant convolutional networks for medical image
  analysis.
\newblock \emph{ArXiv}, abs/1804.03393, 2018.

\bibitem[Bogatskiy et~al.(2020)Bogatskiy, Anderson, Roussi, Miller, and
  Kondor]{BogatskiyLorentz2020}
Alexander Bogatskiy, Brandon Anderson, Marwah Roussi, David Miller, and Risi
  Kondor.
\newblock Lorentz group equivariant neural network for particle physics.
\newblock In \emph{International Conference on Machine Learning}, pp.\
  992--1002. {PMLR}, 2020.

\bibitem[Brunel et~al.(2004)Brunel, Hakim, Isope, Nadal, and
  Barbour]{brunel2004optimal}
Nicolas Brunel, Vincent Hakim, Philippe Isope, Jean-Pierre Nadal, and Boris
  Barbour.
\newblock Optimal information storage and the distribution of synaptic weights:
  perceptron versus purkinje cell.
\newblock \emph{Neuron}, 43\penalty0 (5):\penalty0 745--757, 2004.

\bibitem[Brunel(2016)]{Brunel2016IsCC}
Nicolas J.-B. Brunel.
\newblock Is cortical connectivity optimized for storing information?
\newblock \emph{Nature Neuroscience}, 19:\penalty0 749--755, 2016.

\bibitem[Bump(2004)]{BumpLie}
Daniel Bump.
\newblock \emph{Lie groups}, volume~8.
\newblock Springer, New York, 2004.
\newblock ISBN 9781475740943.

\bibitem[Chapeton et~al.(2012)Chapeton, Fares, LaSota, and
  Stepanyants]{Chapeton2012EfficientAM}
Julio Chapeton, Tarec Fares, Darin LaSota, and Armen Stepanyants.
\newblock Efficient associative memory storage in cortical circuits of
  inhibitory and excitatory neurons.
\newblock \emph{Proceedings of the National Academy of Sciences}, 109:\penalty0
  E3614 -- E3622, 2012.

\bibitem[Chen et~al.(2020)Chen, Dobriban, and Lee]{ChenAugmentation}
Shuxiao Chen, Edgar Dobriban, and Jane Lee.
\newblock A group-theoretic framework for data augmentation.
\newblock In H.~Larochelle, M.~Ranzato, R.~Hadsell, M.~F. Balcan, and H.~Lin
  (eds.), \emph{Advances in Neural Information Processing Systems}, volume~33,
  pp.\  21321--21333. Curran Associates, Inc., 2020.
\newblock URL
  \url{https://proceedings.neurips.cc/paper/2020/file/f4573fc71c731d5c362f0d7860945b88-Paper.pdf}.

\bibitem[Chung et~al.(2018)Chung, Lee, and
  Sompolinsky]{chungClassificationGeometryGeneral2018}
SueYeon Chung, Daniel~D. Lee, and Haim Sompolinsky.
\newblock Classification and {{Geometry}} of {{General Perceptual Manifolds}}.
\newblock \emph{Physical Review X}, 8\penalty0 (3):\penalty0 031003, July 2018.
\newblock \doi{10.1103/PhysRevX.8.031003}.

\bibitem[Cohen \& Welling(2016)Cohen and Welling]{Cohen2016GroupEC}
Taco Cohen and Max Welling.
\newblock Group equivariant convolutional networks.
\newblock In \emph{ICML}, 2016.

\bibitem[Cohen et~al.(2018)Cohen, Geiger, K{\"o}hler, and
  Welling]{Cohen2018SphericalC}
Taco Cohen, Mario Geiger, Jonas K{\"o}hler, and Max Welling.
\newblock Spherical cnns.
\newblock \emph{ArXiv}, abs/1801.10130, 2018.

\bibitem[Cohen et~al.(2019{\natexlab{a}})Cohen, Weiler, Kicanaoglu, and
  Welling]{Cohen2019GaugeEC}
Taco Cohen, Maurice Weiler, Berkay Kicanaoglu, and Max Welling.
\newblock Gauge equivariant convolutional networks and the icosahedral cnn.
\newblock In \emph{ICML}, 2019{\natexlab{a}}.

\bibitem[Cohen et~al.(2019{\natexlab{b}})Cohen, Geiger, and
  Weiler]{cohenGeneralTheoryEquivariant2019}
Taco~S Cohen, Mario Geiger, and Maurice Weiler.
\newblock A general theory of equivariant {{CNNs}} on homogeneous spaces.
\newblock In H.~Wallach, H.~Larochelle, A.~Beygelzimer, F.~{dAlch{\'e}-Buc},
  E.~Fox, and R.~Garnett (eds.), \emph{Advances in Neural Information
  Processing Systems}, volume~32. {Curran Associates, Inc.},
  2019{\natexlab{b}}.

\bibitem[Cohen et~al.(2020)Cohen, Chung, Lee, and
  Sompolinsky]{cohenSeparabilityGeometryObject2020}
Uri Cohen, SueYeon Chung, Daniel~D. Lee, and Haim Sompolinsky.
\newblock Separability and geometry of object manifolds in deep neural
  networks.
\newblock \emph{Nature Communications}, 11\penalty0 (1):\penalty0 746, February
  2020.
\newblock ISSN 2041-1723.
\newblock \doi{10.1038/s41467-020-14578-5}.

\bibitem[Cover(1965)]{Cover1965GeometricalAS}
T.~Cover.
\newblock Geometrical and statistical properties of systems of linear
  inequalities with applications in pattern recognition.
\newblock \emph{IEEE Trans. Electron. Comput.}, 14:\penalty0 326--334, 1965.

\bibitem[Denker et~al.(1989)Denker, Gardner, Graf, Henderson, Howard, Hubbard,
  Jackel, Baird, and Guyon]{NIPS1988_a97da629}
John Denker, W.~Gardner, Hans Graf, Donnie Henderson, R.~Howard, W.~Hubbard,
  L.~D. Jackel, Henry Baird, and Isabelle Guyon.
\newblock Neural network recognizer for hand-written zip code digits.
\newblock In D.~Touretzky (ed.), \emph{Advances in Neural Information
  Processing Systems}, volume~1. Morgan-Kaufmann, 1989.
\newblock URL
  \url{https://proceedings.neurips.cc/paper/1988/file/a97da629b098b75c294dffdc3e463904-Paper.pdf}.

\bibitem[DiCarlo \& Cox(2007)DiCarlo and Cox]{dicarlo2007untangling}
James~J DiCarlo and David~D Cox.
\newblock Untangling invariant object recognition.
\newblock \emph{Trends in cognitive sciences}, 11\penalty0 (8):\penalty0
  333--341, 2007.

\bibitem[Ecker et~al.(2018)Ecker, Sinz, Froudarakis, Fahey, Cadena, Walker,
  Cobos, Reimer, Tolias, and Bethge]{ecker2018rotation}
Alexander~S Ecker, Fabian~H Sinz, Emmanouil Froudarakis, Paul~G Fahey,
  Santiago~A Cadena, Edgar~Y Walker, Erick Cobos, Jacob Reimer, Andreas~S
  Tolias, and Matthias Bethge.
\newblock A rotation-equivariant convolutional neural network model of primary
  visual cortex.
\newblock \emph{arXiv preprint arXiv:1809.10504}, 2018.

\bibitem[Eismann et~al.(2020)Eismann, Townshend, Thomas, Jagota, Jing, and
  Dror]{Eismann2020HierarchicalRN}
Stephan Eismann, Raphael J.~L. Townshend, Nathaniel Thomas, Milind Jagota,
  Bowen Jing, and Ron~O. Dror.
\newblock Hierarchical, rotation‐equivariant neural networks to select
  structural models of protein complexes.
\newblock \emph{Proteins: Structure}, 89:\penalty0 493 -- 501, 2020.

\bibitem[Elesedy \& Zaidi(2021)Elesedy and Zaidi]{ElesedyGeneralization}
Bryn Elesedy and Sheheryar Zaidi.
\newblock Provably strict generalisation benefit for equivariant models.
\newblock In Marina Meila and Tong Zhang (eds.), \emph{Proceedings of the 38th
  International Conference on Machine Learning, {ICML} 2021, 18-24 July 2021,
  Virtual Event}, volume 139 of \emph{Proceedings of Machine Learning
  Research}, pp.\  2959--2969. {PMLR}, 2021.
\newblock URL \url{http://proceedings.mlr.press/v139/elesedy21a.html}.

\bibitem[Esteves et~al.(2018)Esteves, Allen-Blanchette, Makadia, and
  Daniilidis]{Esteves2018LearningSE}
Carlos Esteves, Christine Allen-Blanchette, Ameesh Makadia, and Kostas
  Daniilidis.
\newblock Learning so(3) equivariant representations with spherical cnns.
\newblock In \emph{ECCV}, 2018.

\bibitem[Finzi et~al.(2020)Finzi, Stanton, Izmailov, and
  Wilson]{Finzi2020GeneralizingCN}
Marc Finzi, Samuel Stanton, Pavel Izmailov, and Andrew~Gordon Wilson.
\newblock Generalizing convolutional neural networks for equivariance to lie
  groups on arbitrary continuous data.
\newblock In \emph{ICML}, 2020.

\bibitem[Finzi et~al.(2021)Finzi, Welling, and Wilson]{Finzi2021APM}
Marc Finzi, Max Welling, and Andrew~Gordon Wilson.
\newblock A practical method for constructing equivariant multilayer
  perceptrons for arbitrary matrix groups.
\newblock \emph{ArXiv}, abs/2104.09459, 2021.

\bibitem[Froudarakis et~al.(2020)Froudarakis, Cohen, Diamantaki, Walker,
  Reimer, Berens, Sompolinsky, and Tolias]{froudarakis2020object}
Emmanouil Froudarakis, Uri Cohen, Maria Diamantaki, Edgar~Y Walker, Jacob
  Reimer, Philipp Berens, Haim Sompolinsky, and Andreas~S Tolias.
\newblock Object manifold geometry across the mouse cortical visual hierarchy.
\newblock \emph{bioRxiv}, 2020.

\bibitem[Fulton \& Harris(2004)Fulton and Harris]{FultonHarris}
William Fulton and Joe Harris.
\newblock \emph{Representation Theory: A First Course}.
\newblock Readings in Mathematics. Springer-Verlag New York, New York, 2004.
\newblock ISBN 9780387974958.

\bibitem[Gardner(1987)]{Gardner1987MaximumSC}
E.~Gardner.
\newblock Maximum storage capacity in neural networks.
\newblock \emph{EPL}, 4:\penalty0 481--485, 1987.

\bibitem[Gardner(1988)]{Gardner1988TheSO}
E.~Gardner.
\newblock The space of interactions in neural network models.
\newblock \emph{Journal of Physics A}, 21:\penalty0 257--270, 1988.

\bibitem[Gordon et~al.(2020)Gordon, Lopez-Paz, Baroni, and
  Bouchacourt]{Gordon2020PermutationEM}
Jonathan Gordon, David Lopez-Paz, Marco Baroni, and Diane Bouchacourt.
\newblock Permutation equivariant models for compositional generalization in
  language.
\newblock In \emph{ICLR}, 2020.

\bibitem[Hafting et~al.(2005)Hafting, Fyhn, Molden, Moser, and
  Moser]{haftingMicrostructureSpatialMap2005}
Torkel Hafting, Marianne Fyhn, Sturla Molden, May-Britt Moser, and Edvard~I.
  Moser.
\newblock Microstructure of a spatial map in the entorhinal cortex.
\newblock \emph{Nature}, 436\penalty0 (7052):\penalty0 801--806, August 2005.
\newblock ISSN 0028-0836, 1476-4687.
\newblock \doi{10.1038/nature03721}.

\bibitem[Hartford et~al.(2018)Hartford, Graham, Leyton-Brown, and
  Ravanbakhsh]{Hartford2018DeepMO}
Jason~S. Hartford, Devon~R. Graham, Kevin Leyton-Brown, and Siamak Ravanbakhsh.
\newblock Deep models of interactions across sets.
\newblock In \emph{ICML}, 2018.

\bibitem[Hertz et~al.(2018)Hertz, Krogh, and Palmer]{hertz2018introduction}
John Hertz, Anders Krogh, and Richard~G Palmer.
\newblock \emph{Introduction to the theory of neural computation}.
\newblock CRC Press, 2018.

\bibitem[Huang(2003)]{Huang2003LearningCA}
Guangbin Huang.
\newblock Learning capability and storage capacity of two-hidden-layer
  feedforward networks.
\newblock \emph{IEEE transactions on neural networks}, 14 2:\penalty0 274--81,
  2003.

\bibitem[Keriven \& Peyr\'{e}(2019)Keriven and Peyr\'{e}]{KerivenUniversal2019}
Nicolas Keriven and Gabriel Peyr\'{e}.
\newblock Universal invariant and equivariant graph neural networks.
\newblock In H.~Wallach, H.~Larochelle, A.~Beygelzimer, F.~d\textquotesingle
  Alch\'{e}-Buc, E.~Fox, and R.~Garnett (eds.), \emph{Advances in Neural
  Information Processing Systems}, volume~32. Curran Associates, Inc., 2019.
\newblock URL
  \url{https://proceedings.neurips.cc/paper/2019/file/ea9268cb43f55d1d12380fb6ea5bf572-Paper.pdf}.

\bibitem[Klicpera et~al.(2020)Klicpera, Gro{\ss}, and
  G{\"u}nnemann]{Klicpera2020DirectionalMP}
Johannes Klicpera, Janek Gro{\ss}, and Stephan G{\"u}nnemann.
\newblock Directional message passing for molecular graphs.
\newblock \emph{ArXiv}, abs/2003.03123, 2020.

\bibitem[Kondor \& Trivedi(2018)Kondor and Trivedi]{Kondor2018OnTG}
Risi Kondor and Shubhendu Trivedi.
\newblock On the generalization of equivariance and convolution in neural
  networks to the action of compact groups.
\newblock In \emph{ICML}, 2018.

\bibitem[Kondor et~al.(2018{\natexlab{a}})Kondor, Lin, and
  Trivedi]{Kondor2018ClebschGordanNA}
Risi Kondor, Zhen Lin, and Shubhendu Trivedi.
\newblock Clebsch-gordan nets: a fully fourier space spherical convolutional
  neural network.
\newblock In \emph{NeurIPS}, 2018{\natexlab{a}}.

\bibitem[Kondor et~al.(2018{\natexlab{b}})Kondor, Son, Pan, Anderson, and
  Trivedi]{Kondor2018CovariantCN}
Risi Kondor, Hy~Truong Son, Horace Pan, Brandon~M. Anderson, and Shubhendu
  Trivedi.
\newblock Covariant compositional networks for learning graphs.
\newblock \emph{ArXiv}, abs/1801.02144, 2018{\natexlab{b}}.

\bibitem[Kowalczyk(1997)]{Kowalczyk1997EstimatesOS}
A.~Kowalczyk.
\newblock Estimates of storage capacity of multilayer perceptron with threshold
  logic hidden units.
\newblock \emph{Neural networks : the official journal of the International
  Neural Network Society}, 10 8:\penalty0 1417--1433, 1997.

\bibitem[Krizhevsky(2009)]{krizhevskyLearningMultipleLayers2009}
Alex Krizhevsky.
\newblock Learning multiple layers of features from tiny images.
\newblock Technical report, 2009.

\bibitem[Lanore et~al.(2021)Lanore, {Cayco-Gajic}, Gurnani, Coyle, and
  Silver]{lanoreCerebellarGranuleCell2021}
Frederic Lanore, N.~Alex {Cayco-Gajic}, Harsha Gurnani, Diccon Coyle, and
  R.~Angus Silver.
\newblock Cerebellar granule cell axons support high-dimensional
  representations.
\newblock \emph{Nature Neuroscience}, June 2021.
\newblock ISSN 1097-6256, 1546-1726.
\newblock \doi{10.1038/s41593-021-00873-x}.

\bibitem[LeCun et~al.(1989)LeCun, Boser, Denker, Henderson, Howard, Hubbard,
  and Jackel]{LeCun1989BackpropagationAT}
Yann~Andr{\'e} LeCun, Bernhard~E. Boser, John~S. Denker, Donnie Henderson,
  Richard~E. Howard, Wayne~E. Hubbard, and Lawrence~D. Jackel.
\newblock Backpropagation applied to handwritten zip code recognition.
\newblock \emph{Neural Computation}, 1:\penalty0 541--551, 1989.

\bibitem[Liboff(2004)]{LiboffGOT}
Richard~L. Liboff.
\newblock \emph{Primer for Point and Space Groups}.
\newblock Undergraduate Texts in Contemporary Physics. Springer, New York,
  2004.
\newblock ISBN 9780387402482.

\bibitem[liukuang(2017)]{kuangliu2021}
liukuang.
\newblock pytorch-cifar.
\newblock \url{https://github.com/kuangliu/pytorch-cifar}, 2017.

\bibitem[Lyle et~al.(2020)Lyle, van~der Wilk, Kwiatkowska, Gal, and
  Bloem-Reddy]{Lyle2020OnTB}
Clare Lyle, Mark van~der Wilk, Marta~Z. Kwiatkowska, Yarin Gal, and Benjamin
  Bloem-Reddy.
\newblock On the benefits of invariance in neural networks.
\newblock \emph{ArXiv}, abs/2005.00178, 2020.

\bibitem[Mackey(1970)]{Mackey1970}
George~W. Mackey.
\newblock \emph{Induced Representations of Locally Compact Groups and
  Applications}, pp.\  132--166.
\newblock Springer Berlin Heidelberg, Berlin, Heidelberg, 1970.
\newblock ISBN 978-3-642-48272-4.
\newblock \doi{10.1007/978-3-642-48272-4_6}.
\newblock URL \url{https://doi.org/10.1007/978-3-642-48272-4_6}.

\bibitem[Marcos et~al.(2018)Marcos, Kellenberger, Lobry, and
  Tuia]{Marcos2018ScaleEI}
Diego Marcos, Benjamin Kellenberger, Sylvain Lobry, and Devis Tuia.
\newblock Scale equivariance in cnns with vector fields.
\newblock \emph{ArXiv}, abs/1807.11783, 2018.

\bibitem[Maron et~al.(2019{\natexlab{a}})Maron, Ben-Hamu, Shamir, and
  Lipman]{Maron2019InvariantAE}
Haggai Maron, Heli Ben-Hamu, Nadav Shamir, and Yaron Lipman.
\newblock Invariant and equivariant graph networks.
\newblock \emph{ArXiv}, abs/1812.09902, 2019{\natexlab{a}}.

\bibitem[Maron et~al.(2019{\natexlab{b}})Maron, Fetaya, Segol, and
  Lipman]{maron19a}
Haggai Maron, Ethan Fetaya, Nimrod Segol, and Yaron Lipman.
\newblock On the universality of invariant networks.
\newblock In \emph{Proceedings of the 36th International Conference on Machine
  Learning}, volume~97, pp.\  4363--4371. PMLR, 2019{\natexlab{b}}.
\newblock URL \url{https://proceedings.mlr.press/v97/maron19a.html}.

\bibitem[Maron et~al.(2020)Maron, Litany, Chechik, and Fetaya]{Maron2020OnLS}
Haggai Maron, Or~Litany, Gal Chechik, and Ethan Fetaya.
\newblock On learning sets of symmetric elements.
\newblock \emph{ArXiv}, abs/2002.08599, 2020.

\bibitem[Naimark \& Stern(1982)Naimark and Stern]{NaimarkStern}
Mark~Aronovich Naimark and Aleksandr~Isaakovich Stern.
\newblock \emph{Theory of group representations}, volume 246.
\newblock Springer, New York, 1982.
\newblock ISBN 9781461381440.

\bibitem[Pastore et~al.(2020)Pastore, Rotondo, Erba, and
  Gherardi]{pastoreStatisticalLearningTheory2020}
Mauro Pastore, Pietro Rotondo, Vittorio Erba, and Marco Gherardi.
\newblock Statistical learning theory of structured data.
\newblock \emph{Physical Review E: Statistical Physics, Plasmas, Fluids, and
  Related Interdisciplinary Topics}, 102\penalty0 (3):\penalty0 032119,
  September 2020.
\newblock \doi{10.1103/PhysRevE.102.032119}.

\bibitem[Pehlevan \& Sengupta(2017)Pehlevan and
  Sengupta]{Pehlevan2017ResourceefficientPH}
Cengiz Pehlevan and Anirvan~M. Sengupta.
\newblock Resource-efficient perceptron has sparse synaptic weight
  distribution.
\newblock \emph{2017 25th Signal Processing and Communications Applications
  Conference (SIU)}, pp.\  1--4, 2017.

\bibitem[Perraudin et~al.(2019)Perraudin, Defferrard, Kacprzak, and
  Sgier]{Perraudin2019DeepSphereES}
Nathanael Perraudin, Micha{\"e}l Defferrard, Tomasz Kacprzak, and Raphael
  Sgier.
\newblock Deepsphere: Efficient spherical convolutional neural network with
  healpix sampling for cosmological applications.
\newblock \emph{Astron. Comput.}, 27:\penalty0 130--146, 2019.

\bibitem[Ravanbakhsh(2020)]{ravanbakhsh20a}
Siamak Ravanbakhsh.
\newblock Universal equivariant multilayer perceptrons.
\newblock In Hal~Daumé III and Aarti Singh (eds.), \emph{Proceedings of the
  37th International Conference on Machine Learning}, volume 119 of
  \emph{Proceedings of Machine Learning Research}, pp.\  7996--8006. PMLR,
  13--18 Jul 2020.
\newblock URL \url{https://proceedings.mlr.press/v119/ravanbakhsh20a.html}.

\bibitem[Rigotti et~al.(2013)Rigotti, Barak, Warden, Wang, Daw, Miller, and
  Fusi]{rigotti2013importance}
Mattia Rigotti, Omri Barak, Melissa~R Warden, Xiao-Jing Wang, Nathaniel~D Daw,
  Earl~K Miller, and Stefano Fusi.
\newblock The importance of mixed selectivity in complex cognitive tasks.
\newblock \emph{Nature}, 497\penalty0 (7451):\penalty0 585--590, 2013.

\bibitem[Rotondo et~al.(2020)Rotondo, Lagomarsino, and
  Gherardi]{rotondoCountingLearnableFunctions2020}
Pietro Rotondo, Marco~Cosentino Lagomarsino, and Marco Gherardi.
\newblock Counting the learnable functions of geometrically structured data.
\newblock \emph{Physical Review Research}, 2\penalty0 (2):\penalty0 023169, May
  2020.
\newblock \doi{10.1103/PhysRevResearch.2.023169}.

\bibitem[Rubin et~al.(2017)Rubin, Abbott, and Sompolinsky]{Rubin2017BalancedEA}
Ran Rubin, L.~F. Abbott, and Haim Sompolinsky.
\newblock Balanced excitation and inhibition are required for high-capacity,
  noise-robust neuronal selectivity.
\newblock \emph{Proceedings of the National Academy of Sciences}, 114:\penalty0
  E9366 -- E9375, 2017.

\bibitem[Sannai et~al.(2019{\natexlab{a}})Sannai, Imaizumi, and
  Kawano]{Sannai2019ImprovedGB}
Akiyoshi Sannai, M.~Imaizumi, and M.~Kawano.
\newblock Improved generalization bounds of group invariant / equivariant deep
  networks via quotient feature spaces.
\newblock 2019{\natexlab{a}}.

\bibitem[Sannai et~al.(2019{\natexlab{b}})Sannai, Takai, and
  Cordonnier]{Sannai2019UniversalAO}
Akiyoshi Sannai, Yuuki Takai, and Matthieu Cordonnier.
\newblock Universal approximations of permutation invariant/equivariant
  functions by deep neural networks.
\newblock \emph{ArXiv}, abs/1903.01939, 2019{\natexlab{b}}.

\bibitem[Satorras et~al.(2021)Satorras, Hoogeboom, Fuchs, Posner, and
  Welling]{Satorras2021EnEN}
Victor~Garcia Satorras, E.~Hoogeboom, F.~Fuchs, I.~Posner, and M.~Welling.
\newblock E(n) equivariant normalizing flows for molecule generation in 3d.
\newblock \emph{ArXiv}, abs/2105.09016, 2021.

\bibitem[Schl{\"a}fli(1950)]{Schlafli1950}
Ludwig Schl{\"a}fli.
\newblock \emph{Theorie der vielfachen Kontinuit{\"a}t}, pp.\  167--387.
\newblock Springer Basel, Basel, 1950.
\newblock ISBN 978-3-0348-4118-4.
\newblock \doi{10.1007/978-3-0348-4118-4_13}.
\newblock URL \url{https://doi.org/10.1007/978-3-0348-4118-4_13}.

\bibitem[Sch{\"o}lkopf \& Smola(2002)Sch{\"o}lkopf and
  Smola]{scholkopfLearningKernelsSupport2002}
Bernhard Sch{\"o}lkopf and Alexander~J. Smola.
\newblock \emph{Learning with {{Kernels}}: Support {{Vector Machines}},
  {{Regularization}}, {{Optimization}}, and {{Beyond}}}.
\newblock Adaptive {{Computation}} and {{Machine Learning}}. {MIT Press}, 2002.
\newblock ISBN 978-0-262-19475-4.

\bibitem[Segol \& Lipman(2020)Segol and Lipman]{Segol2020OnUE}
Nimrod Segol and Yaron Lipman.
\newblock On universal equivariant set networks.
\newblock \emph{ArXiv}, abs/1910.02421, 2020.

\bibitem[Serre(2014)]{serreLinearRepresentationsFinite2014}
Jean-Pierre Serre.
\newblock \emph{Linear Representations of Finite Groups}.
\newblock Graduate {{Texts}} in {{Mathematics}}. {Springer}, {New York}, 2014.
\newblock ISBN 978-1-4684-9460-0 978-1-4684-9458-7.

\bibitem[Seung \& Lee(2000)Seung and Lee]{seung2000manifold}
H~Sebastian Seung and Daniel~D Lee.
\newblock The manifold ways of perception.
\newblock \emph{science}, 290\penalty0 (5500):\penalty0 2268--2269, 2000.

\bibitem[Shawe-Taylor(1991)]{ShaweTaylor1991ThresholdNL}
John Shawe-Taylor.
\newblock Threshold network learning in the presence of equivalences.
\newblock In \emph{NIPS}, 1991.

\bibitem[Shcherbina \& Tirozzi(2003)Shcherbina and
  Tirozzi]{shcherbina_rigorous_2003}
Mariya Shcherbina and Brunello Tirozzi.
\newblock Rigorous {Solution} of the {Gardner} {Problem}.
\newblock \emph{Communications in Mathematical Physics}, 234\penalty0
  (3):\penalty0 383--422, March 2003.
\newblock ISSN 0010-3616, 1432-0916.
\newblock \doi{10.1007/s00220-002-0783-3}.
\newblock URL \url{http://link.springer.com/10.1007/s00220-002-0783-3}.

\bibitem[Shutty \& Wierzynski(2020)Shutty and Wierzynski]{Shutty2020LearningIR}
N.~Shutty and Casimir Wierzynski.
\newblock Learning irreducible representations of noncommutative lie groups.
\newblock \emph{ArXiv}, abs/2006.00724, 2020.

\bibitem[Simonyan \& Zisserman(2015)Simonyan and Zisserman]{Simonyan2015VeryDC}
Karen Simonyan and Andrew Zisserman.
\newblock Very deep convolutional networks for large-scale image recognition.
\newblock \emph{CoRR}, abs/1409.1556, 2015.

\bibitem[Smets et~al.(2020)Smets, Portegies, Bekkers, and
  Duits]{Smets2020PDEbasedGE}
Bart Smets, Jim Portegies, Erik~J. Bekkers, and Remco Duits.
\newblock Pde-based group equivariant convolutional neural networks.
\newblock \emph{ArXiv}, abs/2001.09046, 2020.

\bibitem[Sokolic et~al.(2017)Sokolic, Giryes, Sapiro, and
  Rodrigues]{sokolic17a}
Jure Sokolic, Raja Giryes, Guillermo Sapiro, and Miguel Rodrigues.
\newblock {Generalization Error of Invariant Classifiers}.
\newblock In Aarti Singh and Jerry Zhu (eds.), \emph{Proceedings of the 20th
  International Conference on Artificial Intelligence and Statistics},
  volume~54 of \emph{Proceedings of Machine Learning Research}, pp.\
  1094--1103. PMLR, 20--22 Apr 2017.
\newblock URL \url{https://proceedings.mlr.press/v54/sokolic17a.html}.

\bibitem[Sontag(1997)]{Sontag1997ShatteringAS}
Eduardo Sontag.
\newblock Shattering all sets of k points in general position requires (k 1)/2
  parameters.
\newblock \emph{Neural Computation}, 9:\penalty0 337--348, 1997.

\bibitem[Sosnovik et~al.(2021)Sosnovik, Moskalev, and
  Smeulders]{Sosnovik2021ScaleEI}
I.~Sosnovik, A.~Moskalev, and A.~Smeulders.
\newblock Scale equivariance improves siamese tracking.
\newblock \emph{2021 IEEE Winter Conference on Applications of Computer Vision
  (WACV)}, pp.\  2764--2773, 2021.

\bibitem[Sosnovik et~al.(2020)Sosnovik, Szmaja, and
  Smeulders]{Sosnovik2020ScaleEquivariantSN}
Ivan Sosnovik, Michal Szmaja, and Arnold W.~M. Smeulders.
\newblock Scale-equivariant steerable networks.
\newblock \emph{ArXiv}, abs/1910.11093, 2020.

\bibitem[Sreenivasan \& Fiete(2011)Sreenivasan and
  Fiete]{sreenivasanGridCellsGenerate2011}
Sameet Sreenivasan and Ila Fiete.
\newblock Grid cells generate an analog error-correcting code for singularly
  precise neural computation.
\newblock \emph{Nature Neuroscience}, 14\penalty0 (10):\penalty0 1330--1337,
  October 2011.
\newblock ISSN 1097-6256, 1546-1726.
\newblock \doi{10.1038/nn.2901}.

\bibitem[Townshend et~al.(2021)Townshend, Eismann, Watkins, Rangan, Karelina,
  Das, and Dror]{Townshend2021GeometricDL}
Raphael J.~L. Townshend, Stephan Eismann, Andrew~M. Watkins, Ramya Rangan,
  Maria Karelina, Rhiju Das, and Ron~O. Dror.
\newblock Geometric deep learning of rna structure.
\newblock \emph{Science}, 373:\penalty0 1047 -- 1051, 2021.

\bibitem[Vapnik \& Chervonenkis(1968)Vapnik and
  Chervonenkis]{Vapnik1971ChervonenkisOT}
Vladimir~N. Vapnik and Alexey~Y. Chervonenkis.
\newblock On the uniform convergence of relative frequencies of events to their
  probabilities.
\newblock \emph{Dokl. Akad. Nauk.}, 181\penalty0 (4), 1968.

\bibitem[Veeling et~al.(2018)Veeling, Linmans, Winkens, Cohen, and
  Welling]{Veeling2018RotationEC}
Bastiaan~S. Veeling, Jasper Linmans, Jim Winkens, Taco Cohen, and Max Welling.
\newblock Rotation equivariant cnns for digital pathology.
\newblock \emph{ArXiv}, abs/1806.03962, 2018.

\bibitem[Vershynin(2020)]{Vershynin2020MemoryCO}
Roman Vershynin.
\newblock Memory capacity of neural networks with threshold and rectified
  linear unit activations.
\newblock \emph{SIAM J. Math. Data Sci.}, 2:\penalty0 1004--1033, 2020.

\bibitem[Weiler \& Cesa(2019)Weiler and Cesa]{Weiler2019GeneralES}
Maurice Weiler and Gabriele Cesa.
\newblock General e(2)-equivariant steerable cnns.
\newblock \emph{ArXiv}, abs/1911.08251, 2019.

\bibitem[Weiler et~al.(2018{\natexlab{a}})Weiler, Geiger, Welling, Boomsma, and
  Cohen]{Weiler20183DSC}
Maurice Weiler, Mario Geiger, Max Welling, Wouter Boomsma, and Taco Cohen.
\newblock 3d steerable cnns: Learning rotationally equivariant features in
  volumetric data.
\newblock In \emph{NeurIPS}, 2018{\natexlab{a}}.

\bibitem[Weiler et~al.(2018{\natexlab{b}})Weiler, Hamprecht, and
  Storath]{Weiler2018LearningSF}
Maurice Weiler, Fred~A. Hamprecht, and Martin Storath.
\newblock Learning steerable filters for rotation equivariant cnns.
\newblock \emph{2018 IEEE/CVF Conference on Computer Vision and Pattern
  Recognition}, pp.\  849--858, 2018{\natexlab{b}}.

\bibitem[Weiler et~al.(2021)Weiler, Forr\'e, Verlinde, and
  Welling]{Weiler2021CoordinateIC}
Maurice Weiler, Patrick Forr\'e, Erik~P. Verlinde, and Max Welling.
\newblock Coordinate independent convolutional networks - isometry and gauge
  equivariant convolutions on riemannian manifolds.
\newblock \emph{ArXiv}, abs/2106.06020, 2021.

\bibitem[Wendel(1962)]{Wendel1962API}
J.~G. Wendel.
\newblock A problem in geometric probability.
\newblock \emph{Mathematica Scandinavica}, 11:\penalty0 109--112, 1962.

\bibitem[Wigner(1931)]{wignerGruppentheorieUndIhre1931}
Eugene Wigner.
\newblock \emph{Gruppentheorie Und Ihre {{Anwendung}} Auf Die
  {{Quantenmechanik}} Der {{Atomspektren}}}.
\newblock {Vieweg+Teubner Verlag, Wiesbaden}, first edition, 1931.
\newblock ISBN 978-3-663-00642-8.

\bibitem[Winkels \& Cohen(2019)Winkels and Cohen]{Winkels2019PulmonaryND}
Marysia Winkels and Taco Cohen.
\newblock Pulmonary nodule detection in ct scans with equivariant cnns.
\newblock \emph{Medical image analysis}, 55:\penalty0 15--26, 2019.

\bibitem[Worrall \& Brostow(2018)Worrall and Brostow]{Worrall2018CubeNetET}
Daniel~E. Worrall and Gabriel~J. Brostow.
\newblock Cubenet: Equivariance to 3d rotation and translation.
\newblock In \emph{ECCV}, 2018.

\bibitem[Worrall \& Welling(2019)Worrall and Welling]{Worrall2019DeepSE}
Daniel~E. Worrall and Max Welling.
\newblock Deep scale-spaces: Equivariance over scale.
\newblock \emph{ArXiv}, abs/1905.11697, 2019.

\bibitem[Worrall et~al.(2017)Worrall, Garbin, Turmukhambetov, and
  Brostow]{Worrall2017HarmonicND}
Daniel~E. Worrall, Stephan~J. Garbin, Daniyar Turmukhambetov, and Gabriel~J.
  Brostow.
\newblock Harmonic networks: Deep translation and rotation equivariance.
\newblock \emph{2017 IEEE Conference on Computer Vision and Pattern Recognition
  (CVPR)}, pp.\  7168--7177, 2017.

\bibitem[Yarotsky(2018)]{Yarotsky2018UniversalAO}
Dmitry Yarotsky.
\newblock Universal approximations of invariant maps by neural networks.
\newblock \emph{ArXiv}, abs/1804.10306, 2018.

\bibitem[Yun et~al.(2019)Yun, Sra, and Jadbabaie]{Yun2019SmallRN}
Chulhee Yun, S.~Sra, and A.~Jadbabaie.
\newblock Small relu networks are powerful memorizers: a tight analysis of
  memorization capacity.
\newblock In \emph{NeurIPS}, 2019.

\bibitem[Zaheer et~al.(2017)Zaheer, Kottur, Ravanbakhsh, P{\'o}czos,
  Salakhutdinov, and Smola]{Zaheer2017DeepS}
Manzil Zaheer, Satwik Kottur, Siamak Ravanbakhsh, Barnab{\'a}s P{\'o}czos,
  Ruslan Salakhutdinov, and Alex Smola.
\newblock Deep sets.
\newblock In \emph{NIPS}, 2017.

\end{thebibliography}
